\documentclass{article}

\usepackage{scrlfile}
\PreventPackageFromLoading[]{algorithmic}
\usepackage{algorithm}

\usepackage[noend]{algpseudocode}
\algrenewcommand\algorithmicindent{0.5em}

\usepackage{cancel}
\usepackage{microtype}
\usepackage{graphicx}
\usepackage{subcaption}
\usepackage{booktabs} 
\usepackage[percent]{overpic}  

\usepackage{hyperref}

\newcommand{\E}{\mathbb{E}}
\newcommand{\mmu}{\boldsymbol{\mu}}

\newcommand{\xx}{\mathbf{x}}
\newcommand{\yy}{\mathbf{y}}

\usepackage[accepted]{icml2024}

\usepackage{amsmath}
\usepackage{amssymb}
\usepackage{mathtools}
\usepackage{amsthm}
\usepackage{thm-restate}
\usepackage{wrapfig}

\usepackage[capitalize,noabbrev]{cleveref}

\theoremstyle{plain}
\newtheorem{theorem}{Theorem}[section]

\theoremstyle{definition}
\newtheorem{definition}[theorem]{Definition}

\theoremstyle{remark}

\usepackage[textsize=tiny]{todonotes}

\icmltitlerunning{Neural Networks Learn Statistics of Increasing Complexity}

\begin{document}

\twocolumn[
\icmltitle{Neural Networks Learn Statistics of Increasing Complexity}



\icmlsetsymbol{equal}{*}

\begin{icmlauthorlist}
\icmlauthor{Nora Belrose}{eai}
\icmlauthor{Quintin Pope}{osu}
\icmlauthor{Lucia Quirke}{eai}
\icmlauthor{Alex Mallen}{eai}
\icmlauthor{Brennan Dury}{eai}
\icmlauthor{Xiaoli Fern}{osu}
\end{icmlauthorlist}

\icmlaffiliation{eai}{EleutherAI}
\icmlaffiliation{osu}{Oregon State University}

\icmlcorrespondingauthor{Nora Belrose}{nora@eleuther.ai}

\icmlkeywords{Machine Learning, ICML}

\vskip 0.3in
]



\printAffiliationsAndNotice{} 

\begin{abstract}
The \emph{distributional simplicity bias} (DSB) posits that neural networks learn low-order moments of the data distribution first, before moving on to higher-order correlations. In this work, we present compelling new evidence for the DSB by showing that networks automatically learn to perform well on maximum-entropy distributions whose low-order statistics match those of the training set early in training, then lose this ability later. We also extend the DSB to discrete domains by proving an equivalence between token $n$-gram frequencies and the moments of embedding vectors, and by finding empirical evidence for the bias in LLMs. Finally we use optimal transport methods to surgically edit the low-order statistics of one class of images to match those of another, and show early-training networks treat the edited images as if they were drawn from the target class. Code is available at \url{https://github.com/EleutherAI/features-across-time}.
\end{abstract}

\section{Introduction}
\label{introduction}

Neural networks exhibit a remarkable ability to fit complex datasets while generalizing to unseen data points and distributions. This is especially surprising given that deep networks can perfectly fit random labels \citep{zhang2021understanding}, and it is possible to intentionally ``poison'' networks so that they achieve zero training loss while behaving randomly on a held out test set \citep{huang2020understanding}.

A recently proposed explanation for this phenomenon is the \textbf{distributional simplicity bias (DSB)}: neural networks learn to exploit the lower-order statistics of the input data first-- e.g. mean and (co)variance-- before learning to use its higher-order statistics, such as  (co)skewness or (co)kurtosis. \citet{refinetti2023neural} provide evidence for the DSB by training networks on a sequence of synthetic datasets that act as increasingly precise approximations to the real data, showing that early checkpoints perform about as well on real data as checkpoints trained directly on the real data.

We build on \citet{refinetti2023neural} by inverting their experimental setup. We train our models on real datasets, then test them throughout training on synthetic data that probe the model's reliance on statistics of different orders. We believe this experimental design provides more direct evidence about the generalization behavior of commonly used models and training practices.

Our primary theoretical contributions are to \textbf{(1)} motivate the DSB through a Taylor expansion of the expected loss, \textbf{(2)} propose criteria quantifying whether a model ``uses" statistics up to order $k$ by checking that the model is sensitive to interventions on the first $k$ statistics, while being robust to interventions on higher-order statistics, \textbf{(3)} describe efficient methods of producing synthetic data that let us investigate whether models satisfy the above criteria, and \textbf{(4)} extend the DSB to discrete domains by proving an equivalence between token $n$-gram frequencies and the moments of sequences of embedding vectors.

We use a Taylor series expansion to express a model's expected loss as a sum over the central moments of an evaluation dataset. This connection provides some motivation for the DSB. Specifically, if during training, a network's loss is well approximated by the first $k$ terms of its Taylor expansion, then the model should only be sensitive to statistics up to order $k$, and we argue that earlier terms of the expansion will generally become relevant before later terms.

We describe two intuitive criteria that a model sensitive to statistics up to order $k$ should satisfy: \textbf{(1)} changing the first $k$ statistics of data from class A to match class B should cause the model to classify the modified data as class B, and \textbf{(2)} models should be unaffected by ``deleting" higher-order data statistics. We evaluate whether image classification networks satisfy the above criteria during training through extensive empirical experiments across a variety of network architectures and image datasets.

We evaluate whether the network satisfies criterion \textbf{(1)} by generating synthetic datasets where we ``graft" the means and covariances of one class onto images of another class, and evaluating whether the network's classifies the resulting data as belonging to the target class. We formalize this notion of ``grafting'' statistics with optimal transport theory, using an analytic formula to map samples from one class-conditional distribution to another, while minimizing the expected squared Euclidean distance the samples are moved. We also describe coordinatewise quantile normalization, an approximate first order method of grafting the means of one class onto the images of another.

We evaluate the degree to which networks satisfy criterion \textbf{(2)} by generating synthetic data that match the class-conditional means and covariances, but are otherwise maximum entropy.\footnote{Appealing to the principle of maximum entropy to operationalize the notion of ``deletion" in criterion \textbf{(2)}.} We generate two datasets for this purpose. One dataset comes from sampling from a Gaussian distribution with matching mean and covariances. The other dataset comes from incorporating the constraints on image pixel values. We propose a novel gradient-based optimization method to produce samples from a hypercube-constrained maximum entropy distribution. We additionally describe independent coordinate sampling, a first order method of generating hypercube-constrained maximum entropy samples using only means.


\begin{figure}
    \centering
    \begin{subfigure}[t]{0.3\linewidth}
        \includegraphics[width=\linewidth]{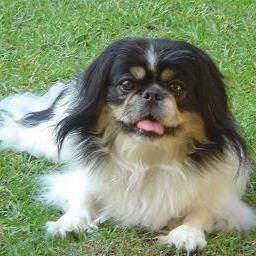}
        \label{fig:original-pekinese}
    \end{subfigure}
    \hfill
    \begin{subfigure}[t]{0.3\linewidth}
        \includegraphics[width=\linewidth]{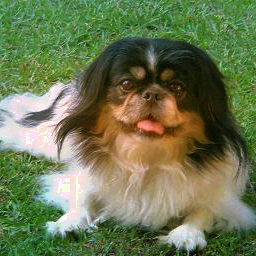}
    \end{subfigure}
    \hfill
    \begin{subfigure}[t]{0.3\linewidth}
        \includegraphics[width=\linewidth]{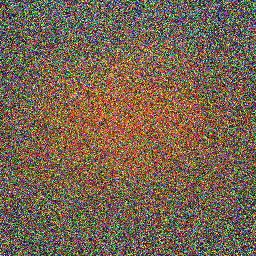}
        \label{fig:fake-goldfish}
    \end{subfigure}
    \caption{\textbf{(left)} Pekinese dog image from the ImageNet training set. \textbf{(center)} Image after quantile normalizing its pixels to match the marginal distribution of the goldfish class on ImageNet. The grass is now a slightly darker shade of green and the dog's fur has a reddish hue. \textbf{(right)} Synthetic ``goldfish'' generated by sampling each pixel independently from its marginal distribution.
    \label{fig:first_order_examples}
    }
\end{figure}

Across models and datasets, we find a common pattern where criteria \textbf{(1)} and \textbf{(2)} hold early in training, with networks largely classifying images according to the means and covariances of the distributions from which they're drawn. But as training progresses, networks become sensitive to higher-order statistics, resulting in a U-shaped loss curve.

We also evaluate EleutherAI's Pythia autoregressive language models \cite{biderman2023pythia} on synthetic data sampled from unigram and bigram models trained on the Pile \cite{gao2020pile}. We find a fascinating ``double descent'' \cite{vallet1989linear, doi:10.1073/pnas.1903070116} phenomenon where models initially mirror the same U-shaped scaling observed in image classifiers, then use in-context learning to achieve even lower loss later in training.

For a thorough review of related work on simplicity biases in machine learning, see Appendix~\ref{app:related-work}.

\section{Theory and Methods}
\label{sec:theory}

\begin{figure}
    \centering

    \hspace{1mm}
    \begin{subfigure}[t]{0.085\linewidth}
        \begin{overpic}[width=\linewidth]{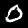}
            \put(-30,40){\rotatebox[origin=c]{90}{\textbf{\tiny MNIST}}}
        \end{overpic}
    \end{subfigure}
    \hfill
    \begin{subfigure}[t]{0.085\linewidth}
        \includegraphics[width=\linewidth]{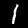}
    \end{subfigure}
    \hfill
    \begin{subfigure}[t]{0.085\linewidth}
        \includegraphics[width=\linewidth]{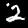}
    \end{subfigure}
    \hfill
    \begin{subfigure}[t]{0.085\linewidth}
        \includegraphics[width=\linewidth]{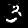}
    \end{subfigure}
    \hfill
    \begin{subfigure}[t]{0.085\linewidth}
        \includegraphics[width=\linewidth]{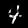}
    \end{subfigure}
    \hfill
    \begin{subfigure}[t]{0.085\linewidth}
        \includegraphics[width=\linewidth]{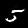}
    \end{subfigure}
    \hfill
    \begin{subfigure}[t]{0.085\linewidth}
        \includegraphics[width=\linewidth]{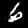}
    \end{subfigure}
    \hfill
    \begin{subfigure}[t]{0.085\linewidth}
        \includegraphics[width=\linewidth]{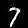}
    \end{subfigure}
    \hfill
    \begin{subfigure}[t]{0.085\linewidth}
        \includegraphics[width=\linewidth]{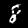}
    \end{subfigure}
    \hfill
    \begin{subfigure}[t]{0.085\linewidth}
        \includegraphics[width=\linewidth]{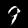}
    \end{subfigure}\\
    \vspace{1mm}
    \vfill

    \hspace{1mm}
    \begin{subfigure}[t]{0.085\linewidth}
        \begin{overpic}[width=\linewidth]{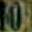}
            \put(-30,40){\rotatebox[origin=c]{90}{\textbf{\tiny SVHN}}}
        \end{overpic}
    \end{subfigure}
    \hfill
    \begin{subfigure}[t]{0.085\linewidth}
        \includegraphics[width=\linewidth]{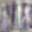}
    \end{subfigure}
    \hfill
    \begin{subfigure}[t]{0.085\linewidth}
        \includegraphics[width=\linewidth]{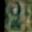}
    \end{subfigure}
    \hfill
    \begin{subfigure}[t]{0.085\linewidth}
        \includegraphics[width=\linewidth]{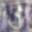}
    \end{subfigure}
    \hfill
    \begin{subfigure}[t]{0.085\linewidth}
        \includegraphics[width=\linewidth]{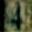}
    \end{subfigure}
    \hfill
    \begin{subfigure}[t]{0.085\linewidth}
        \includegraphics[width=\linewidth]{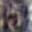}
    \end{subfigure}
    \hfill
    \begin{subfigure}[t]{0.085\linewidth}
        \includegraphics[width=\linewidth]{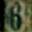}
    \end{subfigure}
    \hfill
    \begin{subfigure}[t]{0.085\linewidth}
        \includegraphics[width=\linewidth]{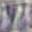}
    \end{subfigure}
    \hfill
    \begin{subfigure}[t]{0.085\linewidth}
        \includegraphics[width=\linewidth]{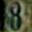}
    \end{subfigure}
    \hfill
    \begin{subfigure}[t]{0.085\linewidth}
        \includegraphics[width=\linewidth]{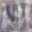}
    \end{subfigure}\\
    \vspace{1mm}
    \vfill

    \hspace{1mm}
    \begin{subfigure}[t]{0.085\linewidth}
        \begin{overpic}[width=\linewidth]{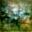}
            \put(-30,40){\rotatebox[origin=c]{90}{\textbf{\tiny CIFAR}}}
        \end{overpic}
    \end{subfigure}
    \hfill
    \begin{subfigure}[t]{0.085\linewidth}
        \includegraphics[width=\linewidth]{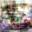}
    \end{subfigure}
    \hfill
    \begin{subfigure}[t]{0.085\linewidth}
        \includegraphics[width=\linewidth]{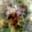}
    \end{subfigure}
    \hfill
    \begin{subfigure}[t]{0.085\linewidth}
        \includegraphics[width=\linewidth]{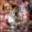}
    \end{subfigure}
    \hfill
    \begin{subfigure}[t]{0.085\linewidth}
        \includegraphics[width=\linewidth]{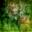}
    \end{subfigure}
    \hfill
    \begin{subfigure}[t]{0.085\linewidth}
        \includegraphics[width=\linewidth]{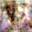}
    \end{subfigure}
    \hfill
    \begin{subfigure}[t]{0.085\linewidth}
        \includegraphics[width=\linewidth]{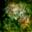}
    \end{subfigure}
    \hfill
    \begin{subfigure}[t]{0.085\linewidth}
        \includegraphics[width=\linewidth]{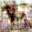}
    \end{subfigure}
    \hfill
    \begin{subfigure}[t]{0.085\linewidth}
        \includegraphics[width=\linewidth]{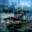}
    \end{subfigure}
    \hfill
    \begin{subfigure}[t]{0.085\linewidth}
        \includegraphics[width=\linewidth]{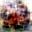}
    \end{subfigure}

    \caption{Non-cherrypicked ``fake'' images produced by maximum entropy sampling using only the first two moments of the class-conditional distributions, and a hypercube constraint. Fake MNIST digits are clearly recognizable, SVHN digits less so, whereas fake CIFAR-10 images look nothing like their respective classes: airplane, automobile, bird, cat, deer, dog, frog, horse, ship, truck.}
    \label{fig:fake-images}
\end{figure}

Let $\mathcal L(\xx)$ denote the loss of a neural network evaluated on input $\xx$. If $\mathcal L_{\theta}$ is analytic\footnote{Famously, the ReLU activation function is not analytic, but it is possible to construct arbitrarily close approximations to ReLU that are analytic \citep[Sec. 4]{hendrycks2016gaussian}.} with an adequate radius of convergence, we can Taylor expand the loss for any given $\xx$ around the mean input $\mmu$ as:
\begin{equation}\label{eq:multi-index}
    \mathcal L(\xx) = \sum_{\alpha \in \mathbb N^d} \frac{(\xx -\mmu)^{\alpha}}{\alpha!} (\partial^{\alpha} \mathcal L)(\mmu),
\end{equation}
where $\alpha$ is a \href{https://en.wikipedia.org/wiki/Multi-index_notation}{multi-index}, or a $d$-tuple assigning an integer to each coordinate of $\xx$. Recall that taking a vector to the power of a multi-index denotes a product of the components of the vector, where each component of the index indicates the multiplicity: e.g. if $\alpha = (1, 4, 6)$, the expression $(\xx -\mmu)^{\alpha}$ denotes the product $(x_1 - \mu_1)(x_2 - \mu_2)^4(x_3 - \mu_3)^6$. Similarly, ($\partial^\alpha \mathcal L$) is shorthand for the mixed partial derivative $\frac{\partial^{11}}{\partial x_1 \partial x_2^4 \partial x_3^6} \mathcal L$. The factorial $\alpha!$ denotes the product of the factorials of the components: $1! \times 4! \times 6! = 17280$.

If $\xx$ is drawn from a distribution with compact support,\footnote{The requirement is slightly weaker than this: we require that the distribution has finite moments of all orders, which is true when the support is compact.} which is true for images and text, we can take the expectation of both sides of Eq.~\ref{eq:multi-index}. This leads to an expression summing over all the \emph{central moments} of $\xx$ multiplied by the corresponding partial derivatives of $\mathcal L$ evaluated at $\mmu$:
\begin{equation}\label{eq:moments}
    \E[\mathcal L(\xx)] = \sum_{\alpha \in \mathbb N^d} \frac{(\partial^{\alpha} \mathcal L)(\mmu)}{\alpha!} \underbrace{\E[(\xx -\mmu)^{\alpha}]}_{\mathrm{central\:moment}}
\end{equation}
Equation~\ref{eq:moments} suggests a close connection between the moments of the data distribution and the expected loss of a neural network evaluated on that distribution.\footnote{Expanding around an an arbitrary point $\mathbf{a}$ would yield an expression containing moments about $\mathbf{a}$, and our analysis would otherwise be unchanged.}

\begin{figure*}
    \centering
    \includegraphics[trim=0 0 0 50, clip, width=0.95\textwidth]{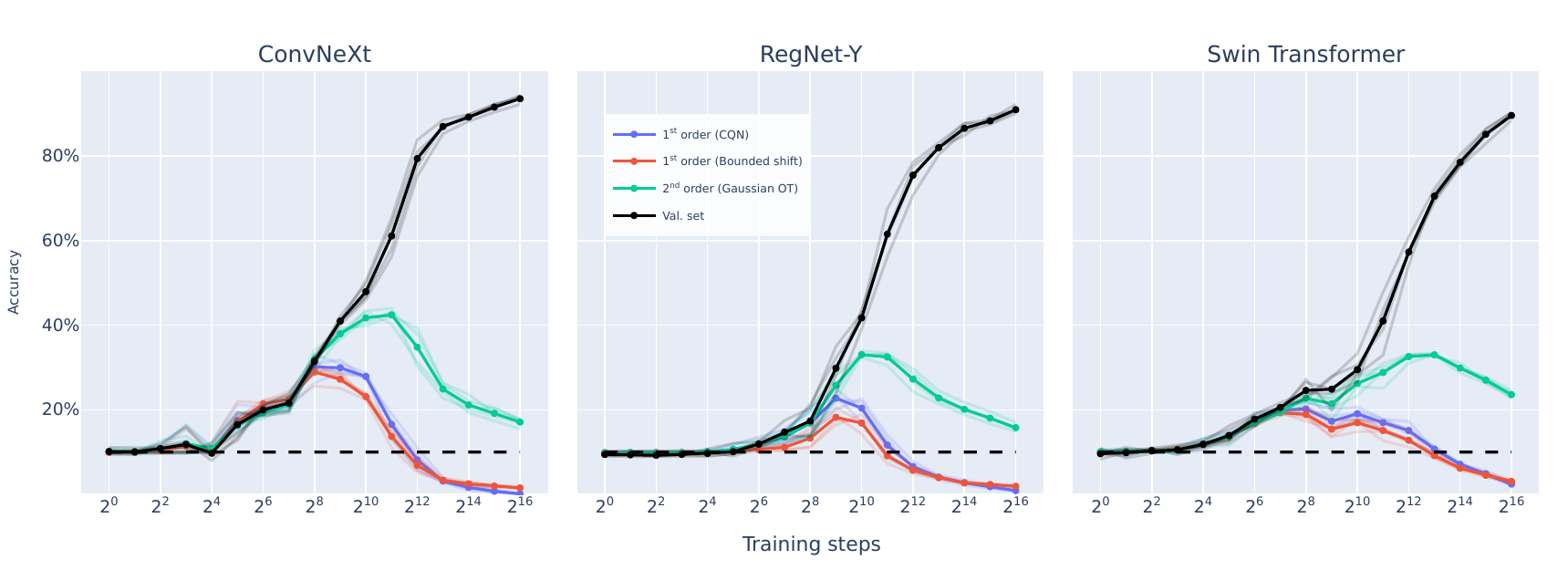}
    \caption{Accuracy of computer vision models when evaluated on images edited with optimal transport maps as described in Sec.~\ref{sec:ot}, using the \emph{target class}, not the source class, as the label. Between roughly $2^4$ and $2^{12}$ training steps, all models classify the CQN-edited images coming from target class, with a peak in accuracy at $2^{9}$.}
    \label{fig:cifar10-ot}
\end{figure*}

\subsection{Intuition}

Without loss of generality, assume the input is constrained to the unit hypercube.\footnote{This is true of images with standard PyTorch preprocessing and one-hot encoded token sequences; other inputs can be rescaled to match this criterion, given our assumption of compact support.} Since every coordinate of $\mathbf{x}$ is no greater than 1, the moments will have magnitudes that monotonically decrease with increasing order; for example, $\E[x_i x_j] \leq \E[x_i] $ for any $i, j \in 1 \ldots d$.

Indeed, we would expect the moment magnitude to decay \emph{exponentially} with order when the coordinates are independent, roughly counterbalancing the exponential increase in the number of distinct moments at higher orders. Assuming the higher derivatives of $\mathcal L$ are reasonably well-behaved at initialization, the $\frac{1}{n!}$ Taylor coefficients will then cause the contribution of higher-order moments to Eq.~\ref{eq:moments} to decay monotonically and factorially fast with order.

As training progresses however, the derivatives of $\mathcal L$ become correlated with the corresponding moments, potentially inflating the magnitude of higher-order terms in Eq.~\ref{eq:moments}. It then seems natural to suppose that the magnitude of higher-order terms will grow in roughly monotonic order-- that is, the second order term will become important first, followed by the third order term, and so on\footnote{Another argument for monotonicity is that earlier terms account for factorially more of the loss at initialization, and are thus plausibly higher-priority targets for gradient descent, until the optimizer is no longer able to easily reduce the loss further by better matching the associated statistics and moves on to higher-order terms.}. 

\subsection{Criteria}

Intuitively, if a model only ``uses'' low-order statistics of the input distribution, this means its behavior should be strongly affected by interventions on the lower-order statistics of the input, but largely unaffected by interventions on the higher-order statistics. More specifically:
\begin{enumerate}
    \item\label{item:grafting} ``Grafting'' the low-order statistics of class $B$ onto class $A$ should cause the model to treat examples from $A$ as if they were from $B$.
    \item\label{item:deletion} ``Deleting'' the information contributed by higher-order statistics should not harm the model's performance.
\end{enumerate}

We operationalize both criteria more precisely below and explain how we produce synthetic data that lets us evaluate the degree to which a given model satisfies each criterion.

\subsection{Optimal Transport}
\label{sec:ot}

\begin{figure*}
    \centering
    \includegraphics[trim=0 140 0 0, clip, width=\textwidth]{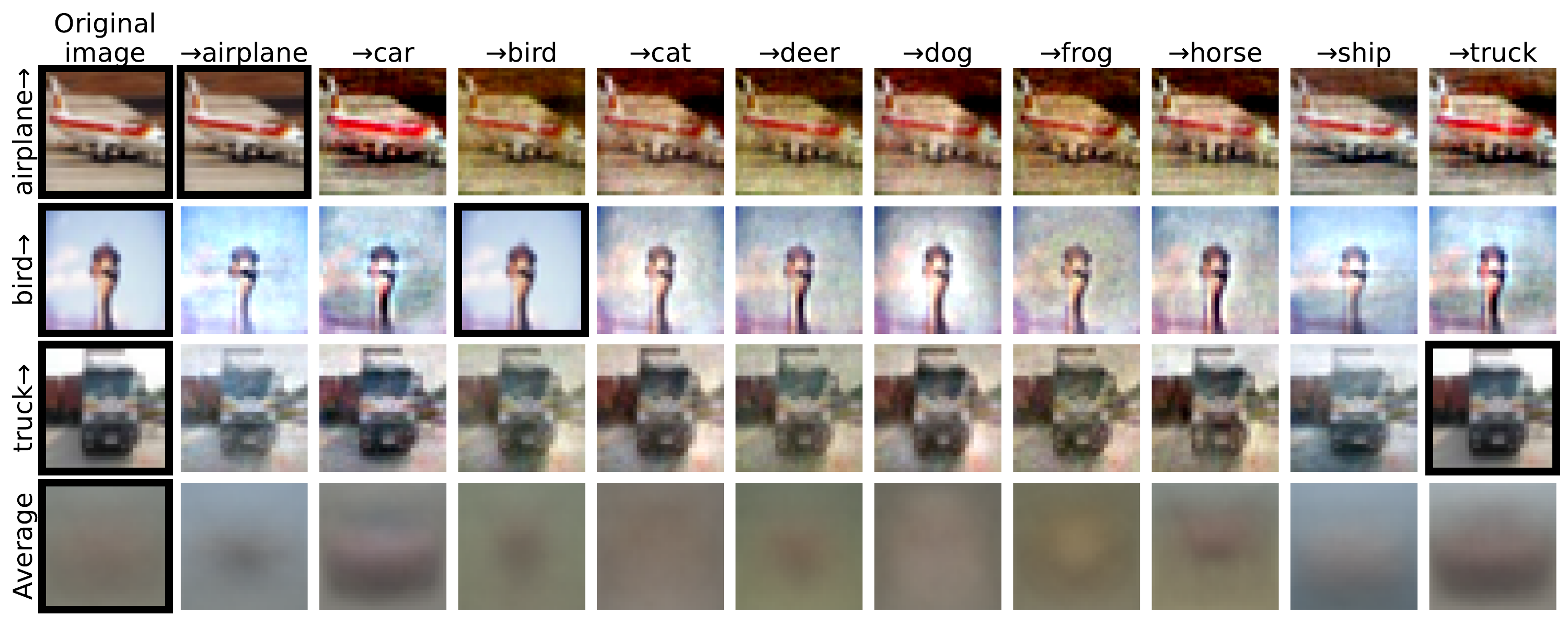}
    \caption{Rows 1-3 show how Gaussian optimal transport affects the example CIFAR-10 airplane, bird and truck images. Each row starts with the original unedited image on the left, with each subsequent column showing the effects of editing that image's first two moments to match the class-conditional distributions of a particular target class \textbf{(top)}.}
    \label{fig:gaussian-ot}
\end{figure*}

We operationalize Criterion~\ref{item:grafting} using optimal transport (OT) theory, which provides tools for transforming samples from one probability distribution into samples from another while minimizing the average distance that samples are moved. We use three OT methods in our experiments: coordinatewise quantile normalization and bounded shift, which primarily affect the first order moments of the distribution, and Gaussian OT, which affects both the first and second-order moments.

\paragraph{Gaussian Optimal Transport}
\label{para:gaussian-ot}

Given two Gaussians $P = \mathcal N(\boldsymbol{\mu}_P, \boldsymbol{\Sigma}_P)$ and $Q = \mathcal N(\boldsymbol{\mu}_Q, \boldsymbol{\Sigma}_Q)$ supported on $\mathbb R^d$, the map $T(\xx) = \mathbf{A}(\xx - \mathbf{m}_P) + \mathbf{m}_Q$ is the optimal transport map from $P$ to $Q$ under the L2 cost function, where
\begin{equation}\label{eq:gaussian-ot}
    \mathbf{A} = \mathbf{\Sigma}_P^{-1/2} \big ( \mathbf{\Sigma}_P^{1/2} \mathbf{\Sigma}_Q \mathbf{\Sigma}_P^{1/2} \big )^{1/2} \mathbf{\Sigma}_P^{-1/2}.
\end{equation}
More generally, if $P$ is an arbitrary distribution with finite second moments, $T(\xx)$ will transport it to a distribution with mean $\boldsymbol{\mu}_Q$ and covariance $\boldsymbol{\Sigma}_Q$, and this map will minimize the cost $\E_{P}[ \| \xx - T(\xx) \|_2^2]$ \citep{dowson1982frechet}.

Given $k$ image classes, each containing tensors of shape $C \times H \times W$, we unroll the tensors into vectors of size $CHW$, then compute their means and covariance matrices,\footnote{Because this is a high-dimensional covariance matrix with dimension only 1-3 times smaller than the sample size, we apply the asymptotically optimal linear shrinkage method proposed by \citet{bodnar2014strong} to improve our estimate of the population covariance and increase numerical stability.} and plug these statistics into Eq.~\ref{eq:gaussian-ot} to get the $k(k - 1)$ optimal transport maps from each class to every other class.


\paragraph{Coordinatewise Quantile Normalization (CQN)}
\label{para:cqn}

Quantile normalization is a technique for making two scalar random variables identical in their statistical properties. When applied \emph{coordinatewise} to the input of a neural network, such as an image, it ensures that the coordinatewise marginals match those of a target distribution, while keeping the correlations between coordinates largely intact, as illustrated by how the edited Pekinese dog image in Fig~\ref{fig:first_order_examples} (\textbf{center}) remains a recognizable dog image. 

CQN works as follows. If a random variable $X$ has cumulative distribution function $F_{X}(x)$, the transformed variable $F_{X}(X)$ will have the standard uniform distribution $\mathrm{Unif}(0, 1)$. Conversely, given variables $U \sim \mathrm{Unif}(0, 1)$ and $Y$, the transformed variable $F_{Y}^{-1}(U)$ will be equal in distribution to $Y$. Composing these transformations together yields quantile normalization. It can be shown that $F_{Y}^{-1} \circ F_{X}$ is the optimal transport map from $X$ to $Y$ for a large class of cost functions \citep[Ch. 2.2]{santambrogio2015optimal}, and is thus ideal for editing the first order statistics of a distribution while minimally perturbing higher-order statistics.


\begin{figure*}
    \centering
    \includegraphics[trim=0 0 0 50, clip, width=\textwidth]
    {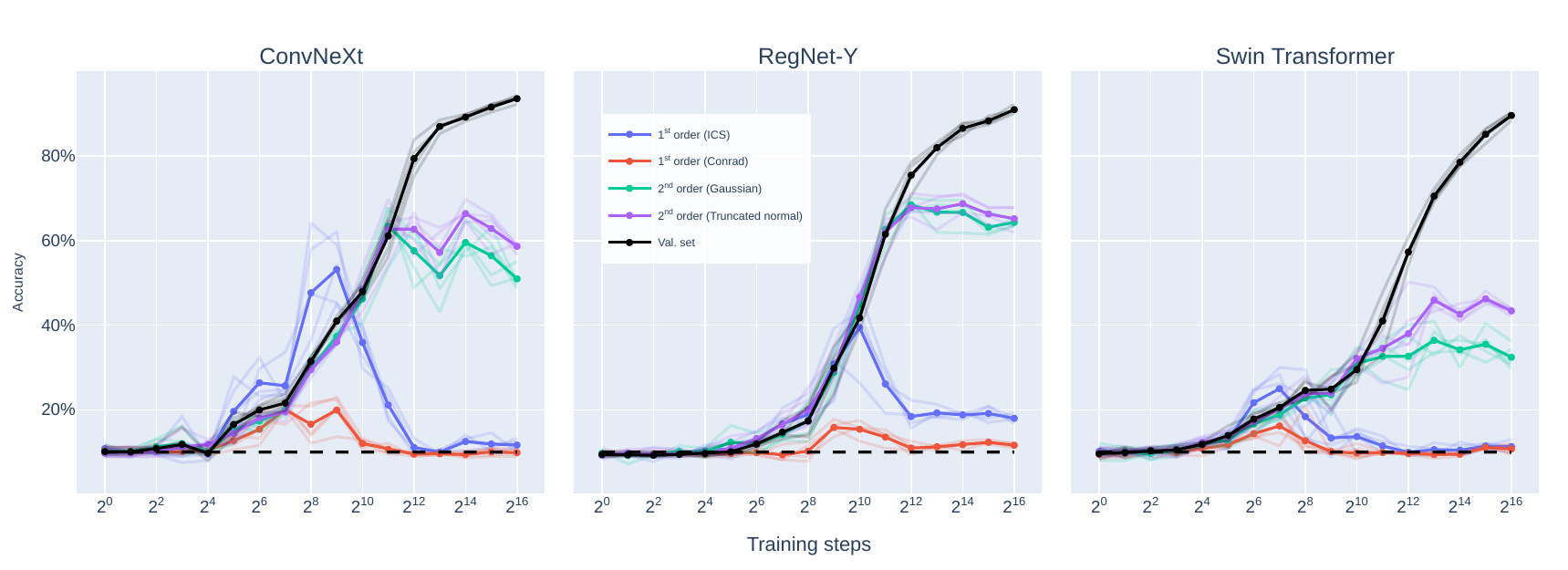}
    \caption{Accuracy of computer vision models being trained on the \emph{standard} CIFAR-10 training set, and being evaluated on maximum-entropy synthetic data with matching statistics of varying order.}
    \label{fig:cifar10-max-ent}
\end{figure*}

\paragraph{Bounded Shift} In Appendix~\ref{app:mean-shift}, we derive an algorithm for shifting the mean of an empirical distribution to a desired value, keeping its support constrained to the interval $[0, 1]$, and minimizing the transport cost. We use this algorithm to graft the mean of one class onto another, while ensuring the pixel intensities of the edited images are valid.

\subsection{Maximum Entropy Sampling}
\label{sec:max-entropy}

We can operationalize Criterion~\ref{item:deletion} using the principle of maximum entropy, which provides a principled method for constructing probability distributions based on ``partial knowledge'' \citep{jaynes1957information}. Here the partial knowledge consists of low-order statistics derived from a training dataset, but we otherwise want to minimize the information content of the higher-order statistics. We therefore want to construct the maximum entropy distribution $P$ consistent with these low-order statistics,\footnote{$P$ can be thought of as the ``least informative'' distribution that satisfies the constraints that its mean and covariance should match those of our original data distribution.} then evaluate a neural network on samples drawn from $P$.

Famously, the maximum entropy distribution supported on $\mathbb R^d$ with known mean $\boldsymbol{\mu}$ and covariance matrix $\boldsymbol{\Sigma}$ is the Gaussian distribution $\mathcal{N}(\boldsymbol{\mu}, \boldsymbol{\Sigma})$. We therefore use Gaussians in many of our experiments in Sec~\ref{sec:image-classification}. In addition to Gaussians, we use hypercube constrained sampling to generate synthetic samples using second and third order statistics, and we use two first-order methods (Conrad sampling and independent coordinate sampling). We explain these methods below.

\paragraph{Hypercube Constraints} One problem with using Gaussians to generate synthetic images is that natural images are constrained to a hypercube: RGB pixel intensities are in the range $[0, 255]$, but nonsingular Gaussian distributions assign positive probability density to all of $\mathbb R^d$, so that a typical sample will often lie outside the hypercube of natural images.\footnote{Strictly speaking, Gaussians also violate the assumption of compact support that we made earlier. In high dimension, though, almost all the probability mass of a Gaussian is contained in the typical set, a compact region near the boundary of an ellipsoid surrounding the mean \citep{carpenter2017typical}.} We might expect neural networks to quickly adapt to such a simple box constraint on the support, so we would like to subject our synthetic images to this constraint.

In the 1D case, the maximum entropy probability density with known mean and variance supported on a finite interval $[a, b]$ has the form $p(x) = \exp(-\lambda_0 - \lambda_1 x - \lambda_2 x^2)$ \citep{dowson1973maximum},\footnote{For some values of the Lagrange multipliers, the formula corresponds to a truncated normal distribution. For sufficiently large variances, the density takes on a U-shape.} but we are unaware of an analytic formula for the Lagrange multipliers in the multidimensional case, making it intractable to solve.\footnote{The log-density of the multidimensional max entropy distribution must be a quadratic form, just like the multivariate normal, but the ``scale'' matrix may not be p.s.d., and solving for the parameters seems challenging.}

We prove in Theorem~\ref{thm:conrad} that the maximum entropy distribution supported on $[0, 1]$ with a fixed mean $\mu \neq \frac{1}{2}$ and unconstrained variance takes the form $p(x) = \frac{b \exp(-b x + b)}{\exp(b) - 1}$, where the parameter $b$ can be found using Newton's method. This formula is not well-known, although an alternative derivation can be found in \citet{conrad2004probability}. To isolate the effect of first-order statistics, we first fit a Conrad distribution to the mean of each coordinate of the images. We then generate synthetic images using \href{https://en.wikipedia.org/wiki/Inverse_transform_sampling}{inverse transform sampling} to produce a value for each coordinate independently. 

\paragraph{Approximate Sampling}\label{para:approximate-sampling} For many sets of constraints, there is no known closed-form solution for the density, precluding sampling techniques like Markov chain Monte Carlo. For these cases, we propose a novel technique for approximate sampling: use gradient-based optimization to directly produce a finite set of samples whose statistics match the desired ones, while maximizing the Kozachenko-Leonenko estimate for the entropy of the implicit population distribution \citep{kozachenko1987sample, sablayrolles2018spreading}. See \href{https://github.com/EleutherAI/features-across-time/blob/d8d7e736b67483904d95db37f7d8ca30fe39e499/scripts/script\_utils/truncated\_normal.py#L17}{\texttt{truncated\_normal.py}} in our codebase for implementation details and Appendix~\ref{app:computational_requirements} for a discussion of computational and memory requirements.

\paragraph{Independent Coordinate Sampling (ICS)}\label{sec:coordinatewise-sampling} In the preceding sections, we decomposed the input distribution into its moments. Another possible decomposition is given by Sklar's theorem, which states that the distribution of any random vector $(X_1, \ldots, X_d)$ is uniquely determined by its coordinatewise marginal CDFs $F_{X_i}(x) = \mathbb P(X_i \leq x)$ and a \emph{copula} function $C: [0, 1]^d \rightarrow [0, 1]$ that combines the marginal CDFs into a multivariate CDF $F_{X}(\mathbf{x}) = \mathbb P(X_1 \leq x_1, \ldots, X_d \leq x_d)$ \citep{sklar1959fonctions}. The maximum entropy copula simply takes the product of the marginal CDFs, and corresponds to a random vector with independent coordinates. We can efficiently sample from this distribution by estimating an empirical CDF for each coordinate, then sampling from each CDF independently.

By constraining the coordinatewise marginals, we ensure that all of the \emph{homogeneous} moments, or moments of the form $\E[(x_i)^n]$, match those of the true data distribution, while the \emph{mixed} moments, e.g. $\E[x_i x_j]$ for $i \neq j$, will generally not match. In high dimension, almost all moments of order greater than one are mixed rather than homogeneous, so ICS matches the first order moments and almost none of the higher order ones. 

\subsection{Discrete Domains}
\label{sec:discrete-domains}

Neural networks use embeddings to convert discrete inputs into vectors of real numbers. The embedding operation can be viewed as a matrix multiplication, wherein the discrete inputs are converted into one-hot vectors we then multiply by the embedding matrix. If the input is a sequence, the result is a sequence of one-hot vectors, or a one-hot matrix.

Just as we unroll images into vectors to compute their moments, we can similarly unroll one-hot matrices to compute their moments. Strikingly, we find that these moments correspond to token $n$-gram frequencies:\footnote{Our formal definition of the term ``$n$-gram statistic'' is non-standard in two respects: first, we include skip-grams (e.g. \textit{the \underline{\hspace{5mm}} dog}, where the underscore is a wildcard token), and second, it is tied to an absolute position in the sequence. However, the Pythia language models we will consider in this paper were trained on chunks of text of uniform length sampled from larger documents \citep{biderman2023pythia}, so the absolute position should not significantly affect the $n$-gram probabilities. We therefore assume in what follows that $n$-gram statistics exhibit translation invariance.}

\begin{restatable}{theorem}{ngrammoments}[$n$-gram statistics are moments]
    \label{thm:ngrams-are-moments}
    Let $\mathcal V^N$ be the set of token sequences of length $N$ drawn from a finite vocabulary $\mathcal V$, let $P$ be a distribution on $\mathcal V^N$, and let $f : \mathcal V^{N} \rightarrow \{0, 1\}^{N \cdot | \mathcal V |}$ be the function that encodes a length-$N$ sequence of tokens as a flattened concatenation of $N$ one-hot vectors of dimension $|\mathcal V|$. Let $f_{\sharp}P$ be the pushforward of $P$ through this one-hot encoding, i.e. its analogue in $\{0, 1\}^{N \cdot | \mathcal V |}$.

    Then every moment of $f_{\sharp}P$ is equal to an $n$-gram statistic of $P$ and vice versa.
\end{restatable}
Furthermore, for a fixed embedding matrix $\mathbf{E}$, two distributions over token sequences that have the same $n$-gram frequencies up to order $k$ will induce distributions over embedding space with the same moments up to $k$:
\begin{restatable}{theorem}{embeddingmoments}[Equal embedding moments]\label{thm:embeddingmoments}
    Let $\mathbf{E} \in \mathbb R^{|\mathcal{V}| \times d}$ be an embedding matrix, and let $P$ and $Q$ be two distributions over $\mathcal V^{N}$. Then if $P$ and $Q$ have the same $n$-gram statistics up to order $k \geq 1$, their embeddings under $\mathbf{E}$ have the same moments up to order $k$.
\end{restatable}
For proofs, see Appendix~\ref{app:discrete-moments}.

Given this equivalence, we can test the DSB in language models with maximum entropy sampling, just like the computer vision case. Given known $n$-gram frequencies up to order $k$, we produce maximum entropy samples using a $k$-gram autoregressive language model. For example, if only bigram frequencies are known, this corresponds to a Markov chain where the distribution of each token depends only on the token immediately preceding it.

\section{Image Classification}
\label{sec:image-classification}

\subsection{Datasets.} Because Gaussian optimal transport (Sec.~\ref{para:gaussian-ot}) requires $O(d^3)$ compute and $O(d^2)$ memory,\footnote{The covariance matrix has $d^2$ elements, where $d$ is the number of pixels, so it is actually $O(n^4)$ in the height or width dimension of the image. We also ran into a software limitation in early experiments where NumPy and PyTorch eigensolvers would crash when fed the very large covariance matrices produced by high-resolution image datasets; see \href{https://github.com/pytorch/pytorch/issues/92141}{PyTorch issue \#92141} for discussion.} we focus on datasets with $32 \times 32$ or $64 \times 64$ resolution images for our primary experiments. Specifically, we examine the popular image classification datasets CIFAR-10 \citep{krizhevsky2009learning}, Fashion MNIST \citep{xiao2017fashion}, MNIST \citep{lecun1998gradient}, and SVHN \citep{netzer2011reading}.

We also build a new image classification dataset, \textbf{CIFARNet}, consisting of 200K images at $64 \times 64$ resolution sampled from ImageNet-21K, using ten coarse-grained classes that roughly match those of CIFAR-10. The larger number of images per class (20K) allows us to get a good estimate of the class-conditional covariance matrices needed for Gaussian optimal transport, which at this resolution contain $(3 \times 64 \times 64)^2 \approx 1.5 \times 10^8$ entries each. See Appendix~\ref{app:cifarnet} for more details on CIFARNet.


\subsection{Architectures.} We focus on state-of-the-art computer vision architectures in our experiments. Specifically, we use ConvNeXt V2 \citep{woo2023convnext} and Swin Transformer V2 \citep{liu2022swin}, which \citet{goldblum2023battle} recently found to have the best performance on a variety of tasks. We train for $2^{16}$ steps with batch size 128, using the AdamW optimizer \citep{loshchilov2018decoupled} with $\beta_1 = 0.9, \beta_2 = 0.95$, and a linear learning rate decay schedule starting at $10^{-3}$ with a warmup of 2000 steps \citep{ma2021adequacy}.\footnote{\label{foot:svhn}We found in early experiments that many models require a lower learning rate to converge on SVHN. We therefore use a learning rate of $10^{-4}$ for ConvNeXt and Swin on this dataset.} For data augmentation, we employ RandAugment \citep{cubuk2020randaugment} followed by random horizontal flips and random crops.

To examine the effect of model scale on our results, we sweep over the Atto, Femto, Pico, Nano, and Tiny sizes for ConvNeXt V2, and we also construct Swin Transformers of roughly analogous sizes.\footnote{The smallest model described in \citet{liu2022swin} is Swin V2 Tiny, which weighs in at 49M parameters. We construct smaller Swin V2 sizes by copying the embedding dimension from the corresponding ConvNeXt V2 size.}

To ensure our results are insensitive to the choice of optimizer and learning rate schedule, we also perform experiments with RegNet-Y \cite{radosavovic2020designing} using SGD with momentum and no LR warmup.

\begin{figure*}[ht]
    \centering
    \includegraphics[width=0.9\textwidth]{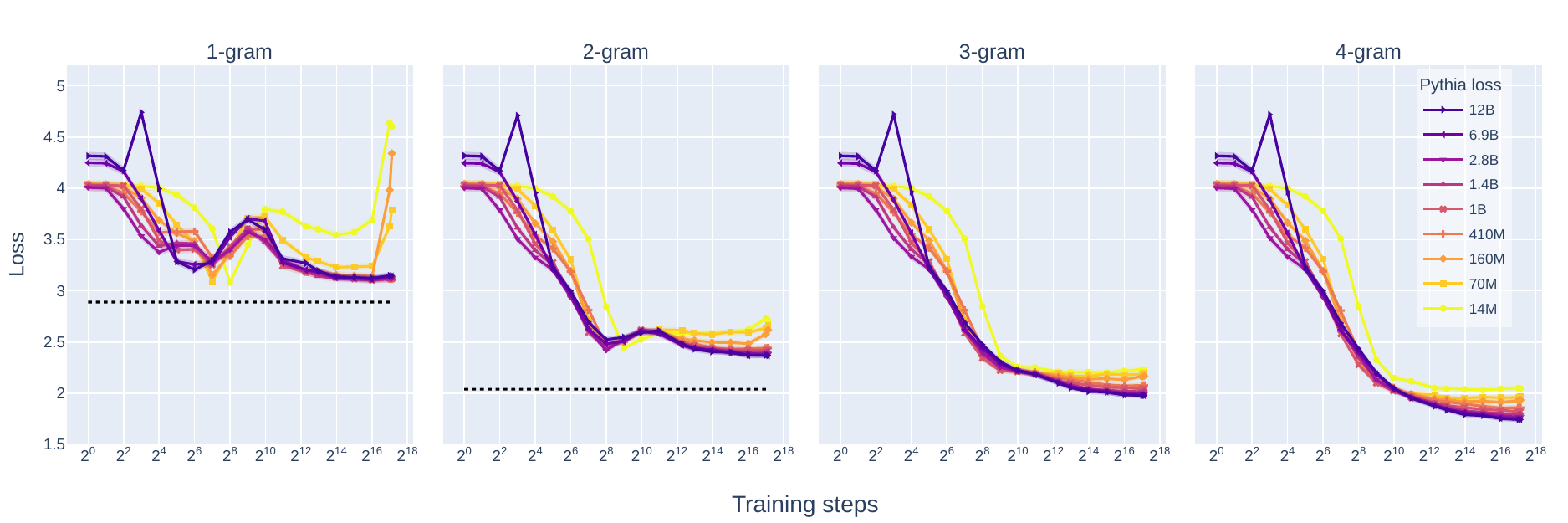}
    \caption{Cross-entropy loss of Pythia suite evaluated on 1- through 4-gram sequences ($N = 4096$.) For comparison, the Shannon entropy of the 1-gram distribution is 2.89 bits per byte (bpb), and is 2.04 bpb for the 2-gram distribution. The first two plots exhibit ``double descent'' scaling: loss reaches a trough between $2^6$ and $2^8$ steps, increases until $2^{10}$ steps, then decreases again as the model learns to match the data generating process in-context (Fig.~\ref{fig:in-context}).} 
    \label{fig:pythia-ngram-loss}
\end{figure*}

\subsection{Results}

We display our results on CIFAR10 in Figures~\ref{fig:cifar10-ot} and~\ref{fig:cifar10-max-ent}, see Appendix~\ref{app:detailed_results} for other datasets.

\paragraph{Optimal transport} We measure the effect of optimal transport interventions by computing the accuracy or loss of the model with respect to the \emph{target class}, rather than the source class. If the intervention is ineffective, we would expect the accuracy to be much lower than the random baseline of 10\%, because the model should confidently classify the images as belonging to the source class. Strikingly, all models we tested get substantially higher than 10\% accuracy w.r.t. the target labels, with ConvNeXt peaking at over 40\% accuracy on 2\textsuperscript{nd} order-edited images after $2^{10}$ training steps.

\paragraph{Maximum entropy sampling} We include four different conditions in our maximum entropy sampling experiments: 1\textsuperscript{st} order (ICS), 2\textsuperscript{nd} order (Gaussian sampling), and both 2\textsuperscript{nd} and 3\textsuperscript{rd} order plus a hypercube constraint.

Overall, we find that accuracy on first order samples peaks earlier in training and has a lower maximum than accuracy on second order samples, followed by the 2\textsuperscript{nd} order hypercube-constrained samples. Remarkably, for some datasets early in training, we find some models achieve \emph{higher} accuracy on the independent pixel samples than they do on images sampled from the real validation set!

\paragraph{Non-monotonicity} Across all datasets, we observe some degree of \emph{non-monotonicity} in the accuracy curves: while models are quite sensitive to low-order moments early in training, they become less sensitive by the end, with accuracy often dipping below the random baseline. The degree of non-monotonicity varies by dataset, however. Very simple datasets like MNIST and Fashion MNIST show very little non-monotonicity, likely because the first and second moments of the data distribution are sufficient to produce very realistic-looking samples (Fig.~\ref{fig:fake-images}).

Overall, we found that model scale has a remarkably small effect on the learning curves, so we display curves averaged over scales in bold, with individual model scales shown as translucent lines. 

\section{Language Modeling}

\begin{figure*}[ht]
    \centering
    \includegraphics[width=0.9\textwidth]{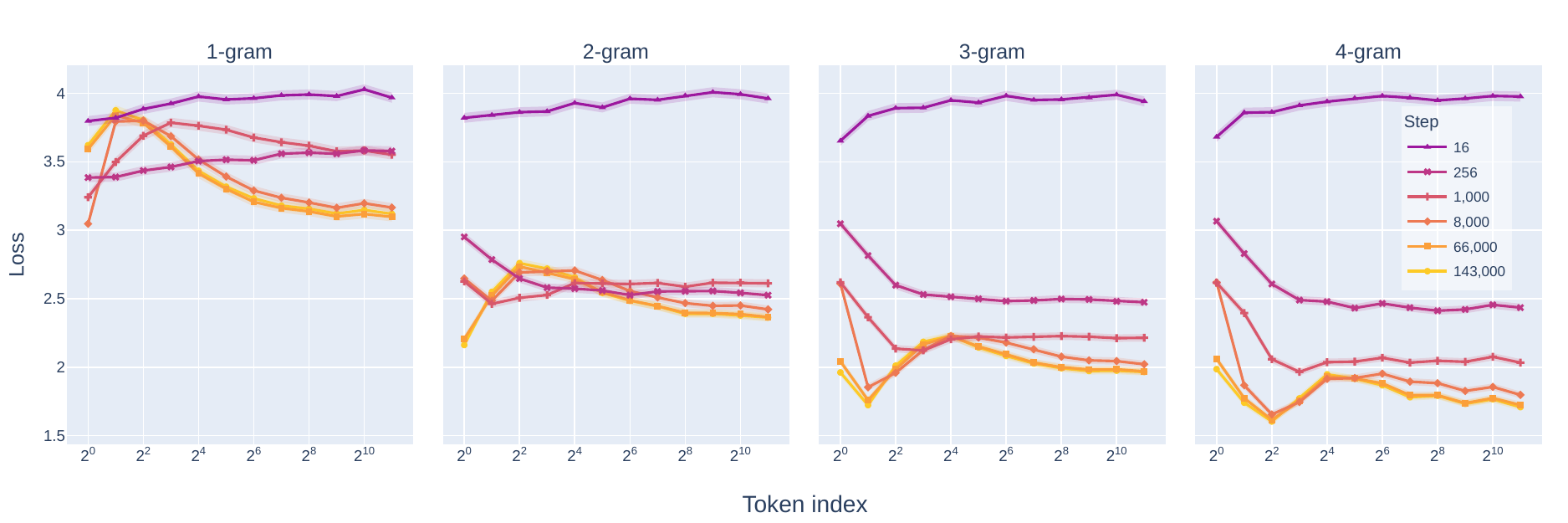}
    \caption{Mean loss over n-gram sequence positions at six Pythia-12B training steps. In-context learning of the maximum entropy bigram sequence occurs after step 8,000. Some in-context learning of the maximum entropy unigram sequence occurs by step 1,000, with more after step 8,000.}
    \label{fig:in-context}
\end{figure*}

To test the distributional simplicity bias in a discrete domain, we study EleutherAI's Pythia language model suite \cite{biderman2023pythia}, for which checkpoints are publicly available at log-spaced intervals throughout training. Model parameter counts range from 14 million to 12 billion.

There are ten independent training runs for Pythia 14M, 70M, and 160M publicly available on the HuggingFace Hub, each using a different random seed. There are also five available seeds for Pythia 410M. We take advantage of these additional runs to examine the effect of random seed on our results. We also trained custom variants of Pythia 14M and 70M with an extended learning rate warmup period of 14,300 steps to isolate the effect of LR warmup.

While we include skip-grams (e.g. \textit{the \underline{\hspace{5mm}} dog}) in our formal definition of $n$-gram frequency (Def.~\ref{def:ngram-statistic}), we do not include them in these experiments for tractability reasons: they would greatly increase the memory and storage requirements of maximum entropy sampling. We hope to explore the effect of skip-gram statistics in future work.

\paragraph{$\boldsymbol{n}$-gram language models} We compute token unigram and bigram frequencies across Pythia's training corpus, the Pile \citep{gao2020pile}, and use these statistics to construct maximum entropy $n$-gram language models. We autoregressively sample sequences of length 2049 from the $n$-gram LMs, and evaluate Pythia's cross-entropy loss on these maximum entropy samples at each checkpoint. We repeat the procedure for 3- and 4-grams using a subset of Pythia's training corpus.
%


Additionally, we evaluate Pythia 12B's cross-entropy loss over each token position in the maximum entropy $n$-gram sequences for six different checkpoints in order to detect the development of in-context learning. If in-context learning is involved in making predictions for these sequences at a training step, cross-entropy loss should decrease over successive token positions in the sequence for that step \cite{olsson2022context}.

\subsection{Results}

We display our results on the Pythia suite in Fig.~\ref{fig:pythia-ngram-loss}. See Appendix~\ref{app:detailed_results} for alternate model seeds and learning rate warmup. Overall we find that the random seed has very little effect on the learning curves, and lengthening the LR warmup period did not consistently affect their overall shape.

\paragraph{$\boldsymbol{n}$-gram sequence loss}

Consistent with the image classification tasks, unigram sequence loss consistently reaches its lowest point before bigram sequence loss and has a higher minimum value.

Across all models, we observe non-monotonicity in the unigram and bigram sequence loss curves, where loss steeply reduces and then increases to a lesser extent. However, unlike in the image classification tasks, the loss reverts to a monotonic regime later in training. We hypothesize that this is caused by the development of in-context learning sufficient to improve n-gram sequence predictions. We observe correlational evidence in the $n$-gram sequence loss over increasing token indices and training steps in Pythia 12B (Fig.~\ref{fig:in-context}), where in-context learning seems to emerge in the same training step where the non-monotonic regime ends. Fascinatingly, smaller models seem to resume the standard ‘U’-shaped loss pattern in the later portions of training.\footnote{Arguably, this pattern applies to models of all sizes on unigram sequences, but the tiny increases in loss for the larger models are within the margin of error for these experiments.}

We speculate that this behavior may arise from a form of ``catastrophic forgetting'', in which all models initially learn low-order $n$-gram statistics, which are eventually eclipsed by more sophisticated features. Larger models have greater representational capacity, and so are better able to retain these early $n$-gram features.

We do not observe non-monoticity in the higher order $n$-gram sequence loss curves. However, the 3-gram loss plateaus at a point consistent with the non-monoticity observed in 1- and 2-grams, suggesting that the effect could be present to a lesser extent.

\paragraph{In-context learning}

We follow \citet{kaplan2020scaling} in defining in-context learning as decreasing loss at increasing token indices. We find that loss is uniform across token positions in early training steps, but slowly decreases at increasing token indices in later steps, consistent with the presence of in-context learning (Fig.~\ref{fig:in-context}). 

We observe an initial increase in loss early in each sequence. This is likely due to the fact unigram sequences are indistinguishable from real sequences at the first position, and bigram model predictions are indistinguishable from real sequences at the first and second positions.


\section{Conclusion}


We propose two criteria that operationalize what it means for models to exploit moments of a given order, then describe methods of generating synthetic data that test whether a network satisfies both criteria, using theoretically grounded approaches relying on optimal transport theory and the principle of maximum entropy. We extend our analysis to discrete sequences by proving an equivalence between $n$-gram statistics and statistical moments.

We find new compelling empirical evidence that neural networks learn to exploit the moments of their input distributions in increasing order, and further find ``double descent'' in the degree to which LMs are able to model sequences sampled from low-order data statistics, driven by in-context learning on longer sequences. Our contributions strengthen the case for the distributional simplicity bias (DSB), refine our understanding of how DSB influences early learning dynamics, and provide a foundation for further investigations into DSB.

\subsection*{Acknowledgements}

We are thankful to Open Philanthropy for \href{https://www.openphilanthropy.org/grants/eleuther-ai-interpretability-research/}{funding this work}. We also thank \href{https://newscience.org/}{New Science} and \href{https://stability.ai/}{Stability AI} for providing computing resources.

\subsection*{Impact statement}

The goal of this work was to advance our understanding of the generalization behavior of neural networks throughout training, in the hope that this will enable the development of more robust and predictable machine learning models.



\bibliography{citations}
\bibliographystyle{icml2024}

\newpage
\appendix
\onecolumn
\section{Additional Related Work}
\label{app:related-work}

Extensive prior work has investigated neural network simplicity bias and learning dynamics. We highlight several prior research directions that usefully contrast our own approach.

\subsection{Simplicity bias}
One common approach studies simplicity biases in the parameter-function maps of neural network architectures. Such explanations posit that neural networks implement favorable priors, meaning that most network parameterizations, under commonly used initialization distributions, that reach good performance on the training data will also generalize to the test data, regardless of specific details about the optimization process used to find such parameterizations. 

\citet{valle2018deep} investigated such architectural simplicity biases by using Gaussian process-based approximations to neural networks \cite{lee2018deep} to estimate the Bayesian posterior produced by randomly sampling neural network parameterizations, conditional on those networks achieving perfect training loss, and showed the resulting posterior correlated well with the odds of SGD-based training finding a given function. \citet{chiang2023loss} validate this perspective by showing that a variety of non-gradient based optimizers, including unbiased sampling of random initializations, are still able to generalize from training to testing data. 

Another approach is to construct a simplified, theoretically tractable model of neural network learning dynamics, then analyzing the resulting model to find which types of functions it predicts networks will be most inclined to learn. The neural tangent kernel \cite{NEURIPS2018_5a4be1fa}, scales network widths to infinity, whereupon networks are limited to performing kernel regression with their initialization kernel. Thus, model inductive biases are determined by the spectrum of the initialization kernel's eigenfunctions, which have strong simplicity biases for commonly used architectures \cite{canatar2021spectral, baratin2021implicit, bietti2019inductive}.

\subsection{Learning order}

\citet{xu2019training} proposed the Frequency Principle, the tendency of neural networks to first fit low-frequency Fourier components of a given target function, before moving on to fit higher frequency components, and empirically demonstrated this tendency on real image classification problems and synthetic datasets. Subsequent works further explored how neural network learning dynamics relate to the representation of training data in the frequency domain \cite{pmlr-v97-rahaman19a, xu2019frequency, pmlr-v119-basri20a, Xu_Zhou_2021}. Our work is similar in that we also aim to connect neural network learning order to simple mathematical properties of the training data, though we use distributional statistics, rather than frequency. 

\citet{choshen-etal-2022-grammar} empirically studied learning dynamics of neural language models by tracking which grammatical patterns different networks learn to model across their training trajectories, and comparing network behavior across training to alternative language modeling approaches, such as $n$-gram models. They found that neural language models initially match the behaviors of unigram and bigram models early in training, then diverge as training progresses. These results are inline with our own findings on learning order in neural language models, and are consistent with a DSB-driven perspective on neural network learning dynamics.


\newpage
\section{CIFARNet dataset}
\label{app:cifarnet}

CIFARNet is based on the Winter 2019 version of ImageNet-21K. We selected the ten synsets from the ImageNet hierarchy which most closely matched the ten CIFAR-10 classes, with a bias toward broader synsets to maximize the dataset size:
\begin{itemize}
    \item Airplane: \texttt{n02691156}
    \item Automobile: \texttt{n02958343}
    \item Bird: \texttt{n01503061}
    \item Cat: \texttt{n02121620}
    \item Deer: \texttt{n02430045}
    \item Dog: \texttt{n02083346}
    \item Frog: \texttt{n01639765}
    \item Horse: \texttt{n02374451}
    \item Ship: \texttt{n04194289}
    \item Truck: \texttt{n04490091}
\end{itemize}
We ensured class balance by randomly sampling 20K images from each synset. Images were directly resized to $64 \times 64$ resolution without center cropping.

\newpage
\section{Detailed experimental results}
\label{app:detailed_results}

\subsection{CIFAR-10}

\begin{figure}[h!]
    \centering
    \includegraphics[width=\textwidth]{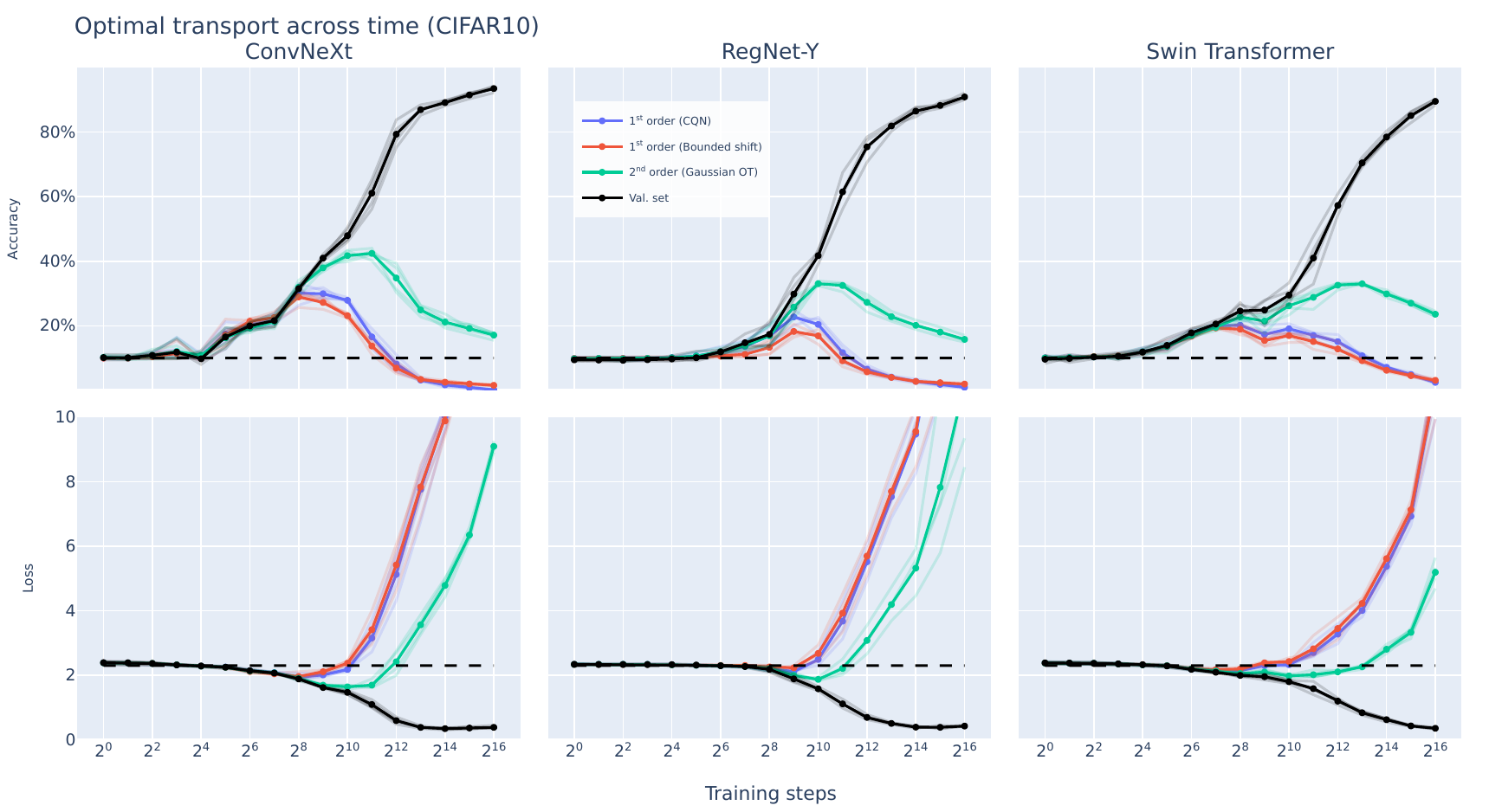}
    \includegraphics[width=\textwidth]{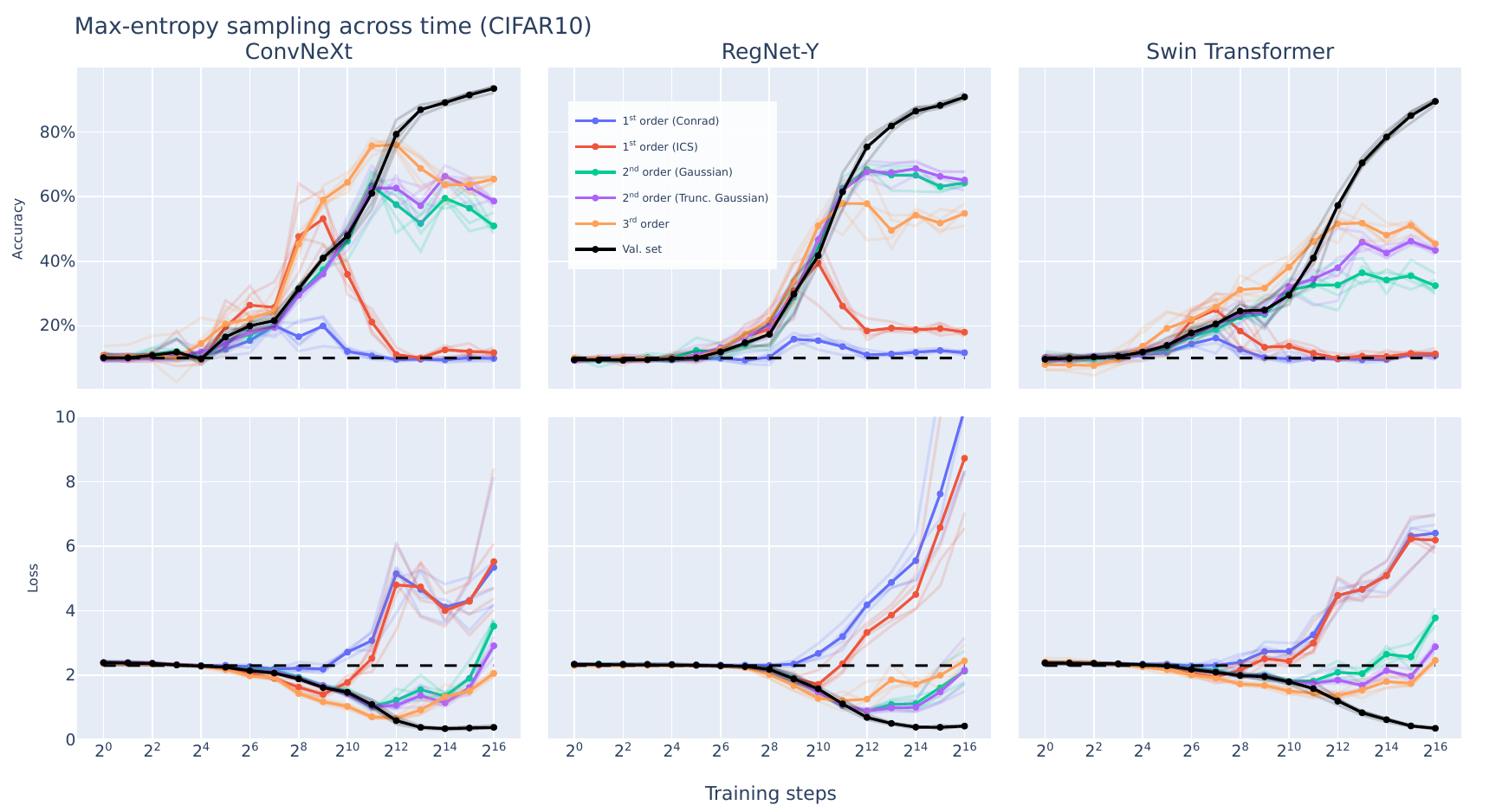}

    \label{fig:cifar10-full}
\end{figure}

\newpage
\subsection{CIFARNet}

\begin{figure}[h!]
    \centering
    \includegraphics[width=\textwidth]{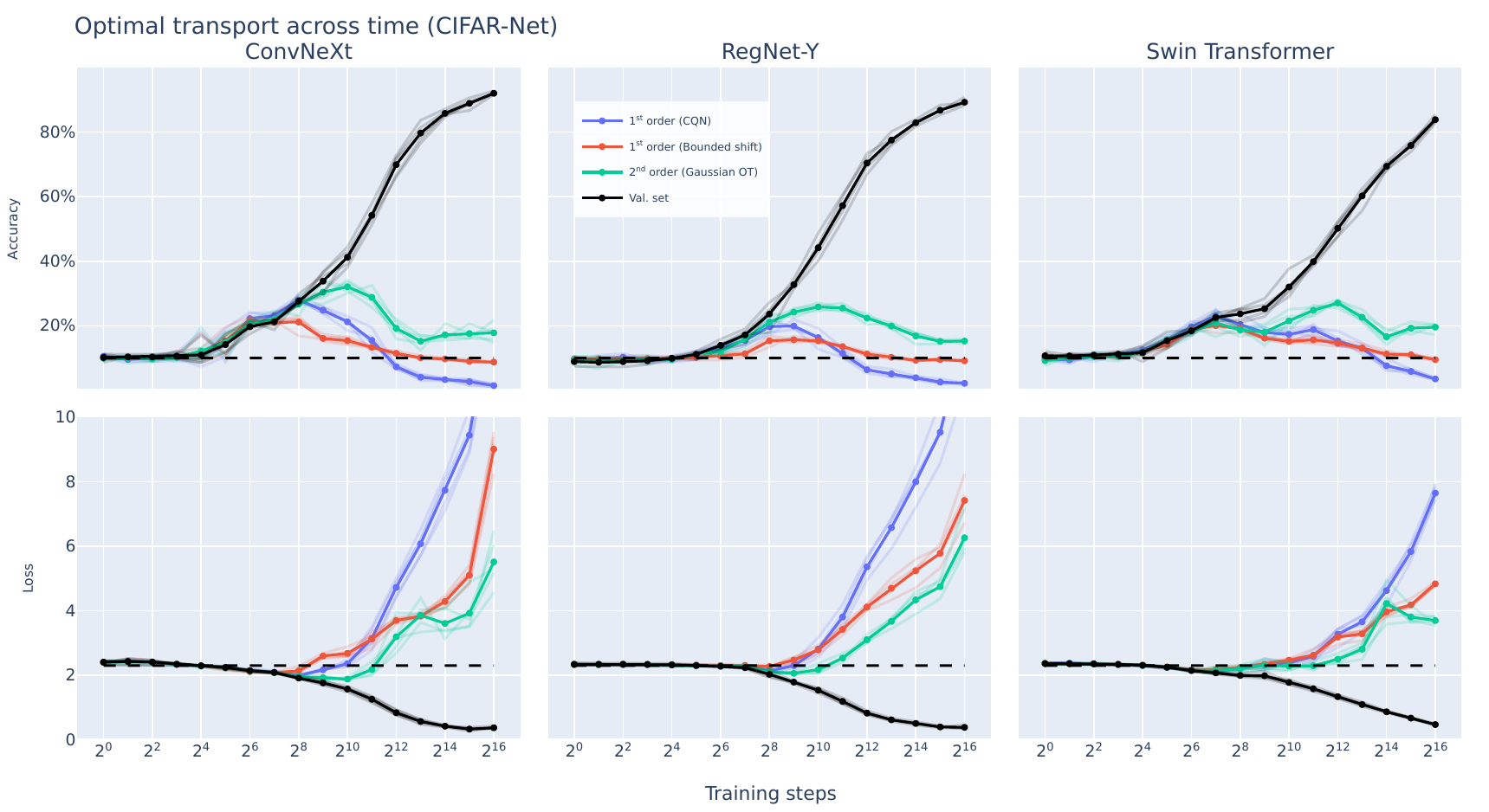}
    \includegraphics[width=\textwidth]{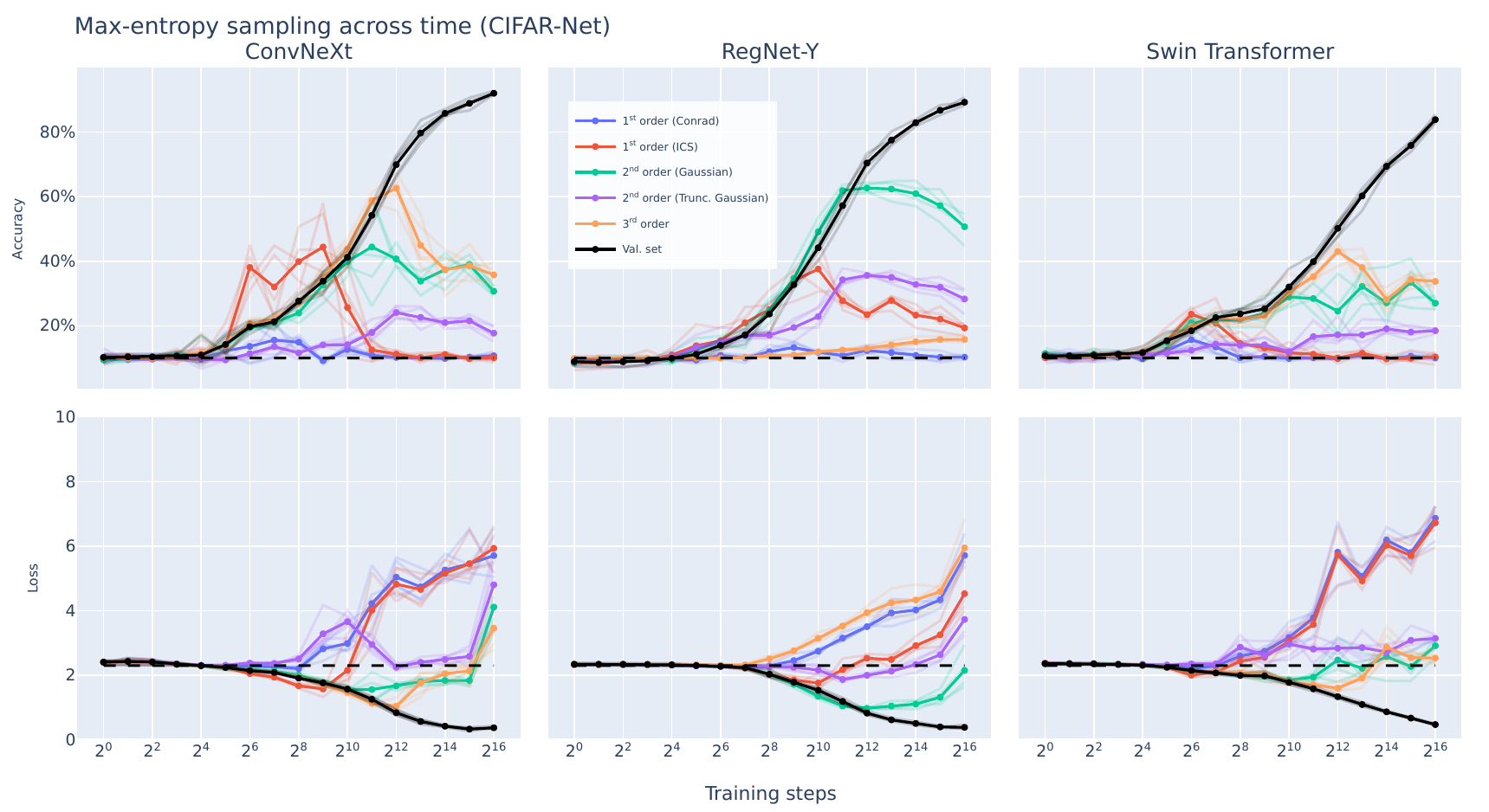}

    \caption{These results qualitatively mirror those of the lower resolution CIFAR-10 dataset (see above), except that the maximum accuracies attained on 2\textsuperscript{nd} order samples are somewhat lower. This may suggest that networks more quickly learn to use higher-order statistics when the input has higher dimensionality.}
    \label{fig:cifarnet}
\end{figure}

\newpage
\subsection{Fashion MNIST}

\begin{figure}[h!]
    \centering
    \includegraphics[width=\textwidth]{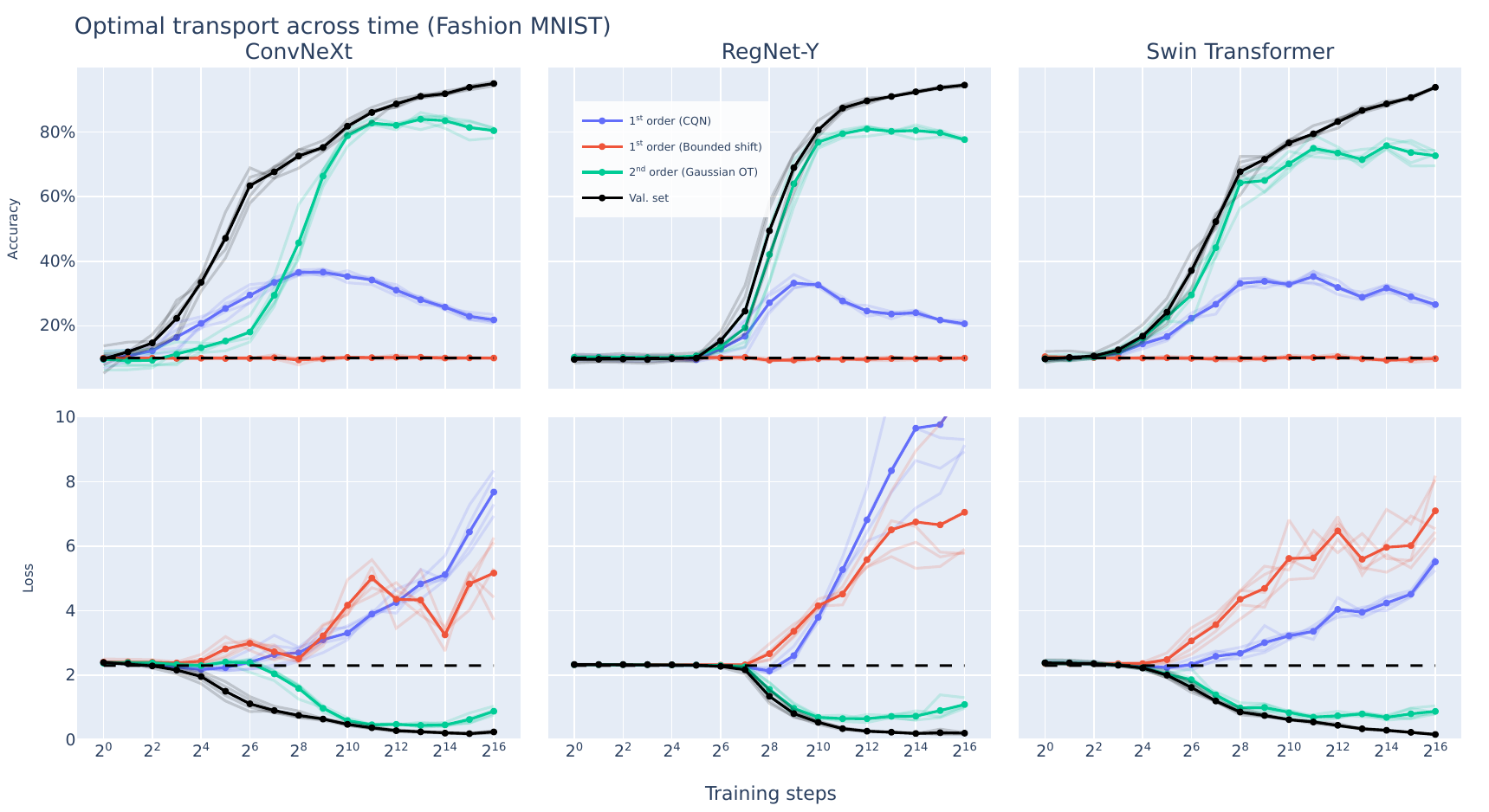}
    \includegraphics[width=\textwidth]{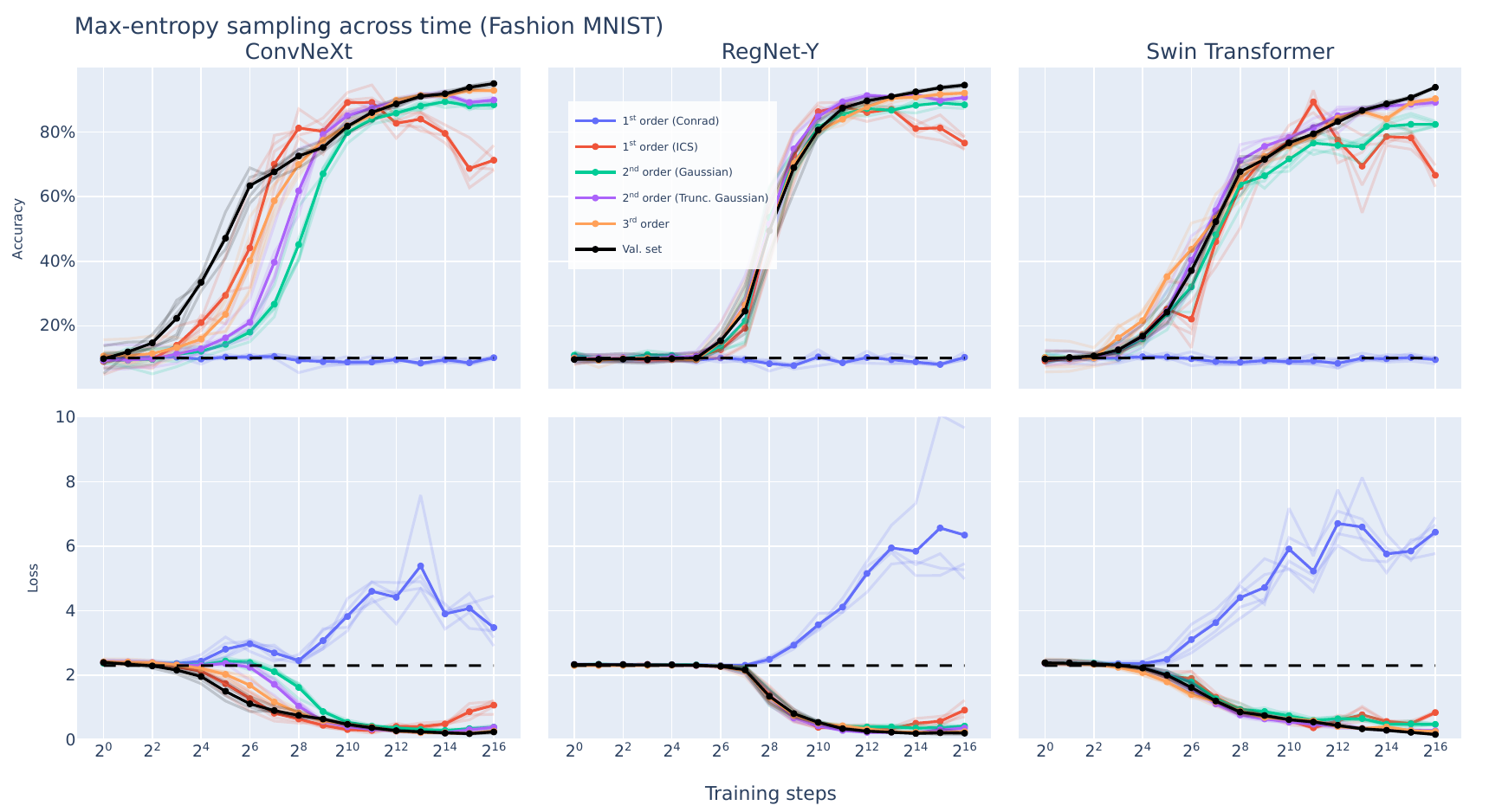}

    \caption{Fashion MNIST learning curves exhibit only a modest degree of non-monotonicity, likely because the first and second moments of the data distribution are sufficient to produce very realistic-looking samples (Fig.~\ref{fig:fake-images})}
    \label{fig:fashion-mnist}
\end{figure}

\newpage
\subsection{MNIST}
\label{app:mnist}

\begin{figure}[h!]
    \centering
    \includegraphics[width=\textwidth]{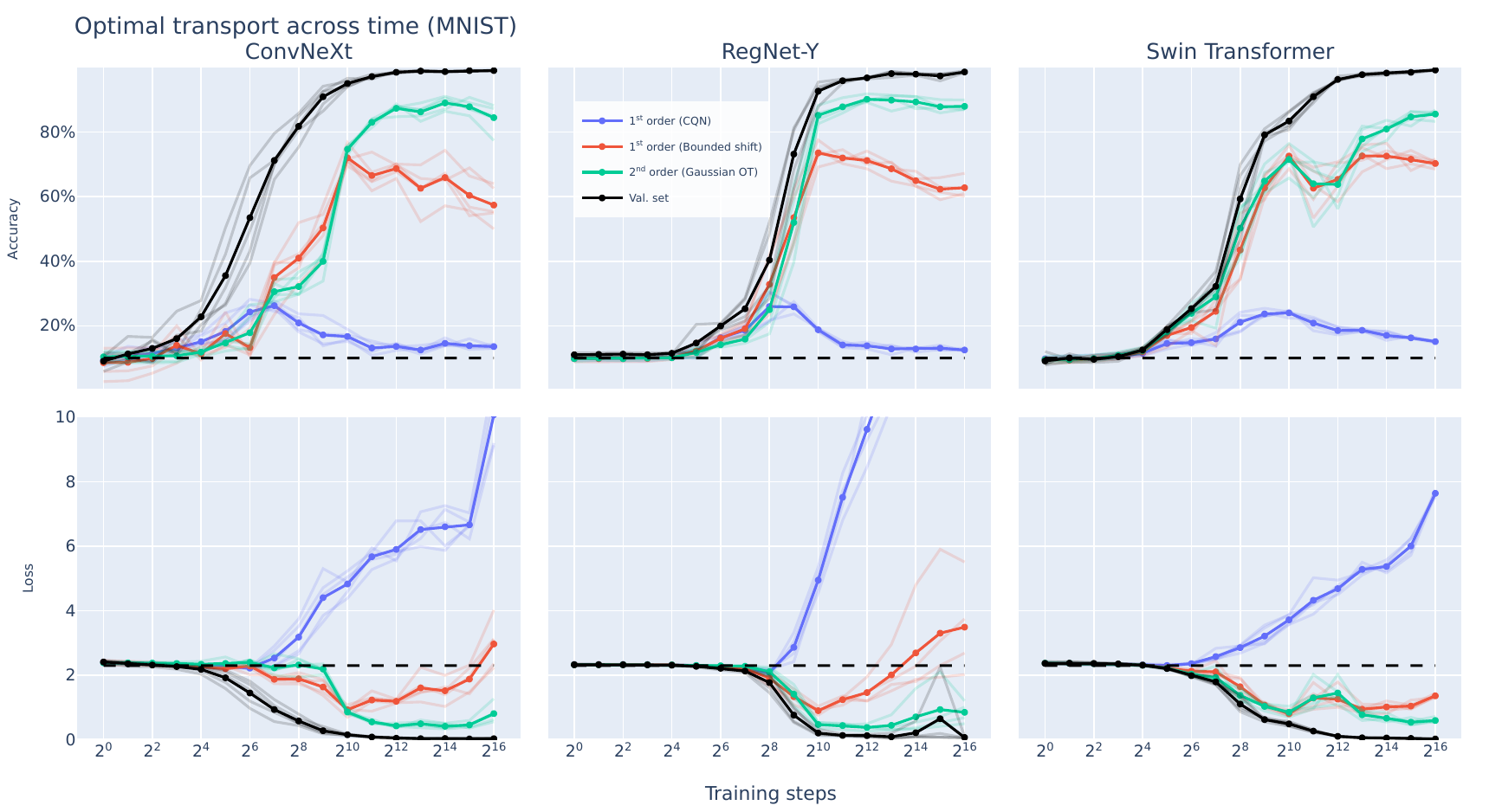}
    \includegraphics[width=\textwidth]{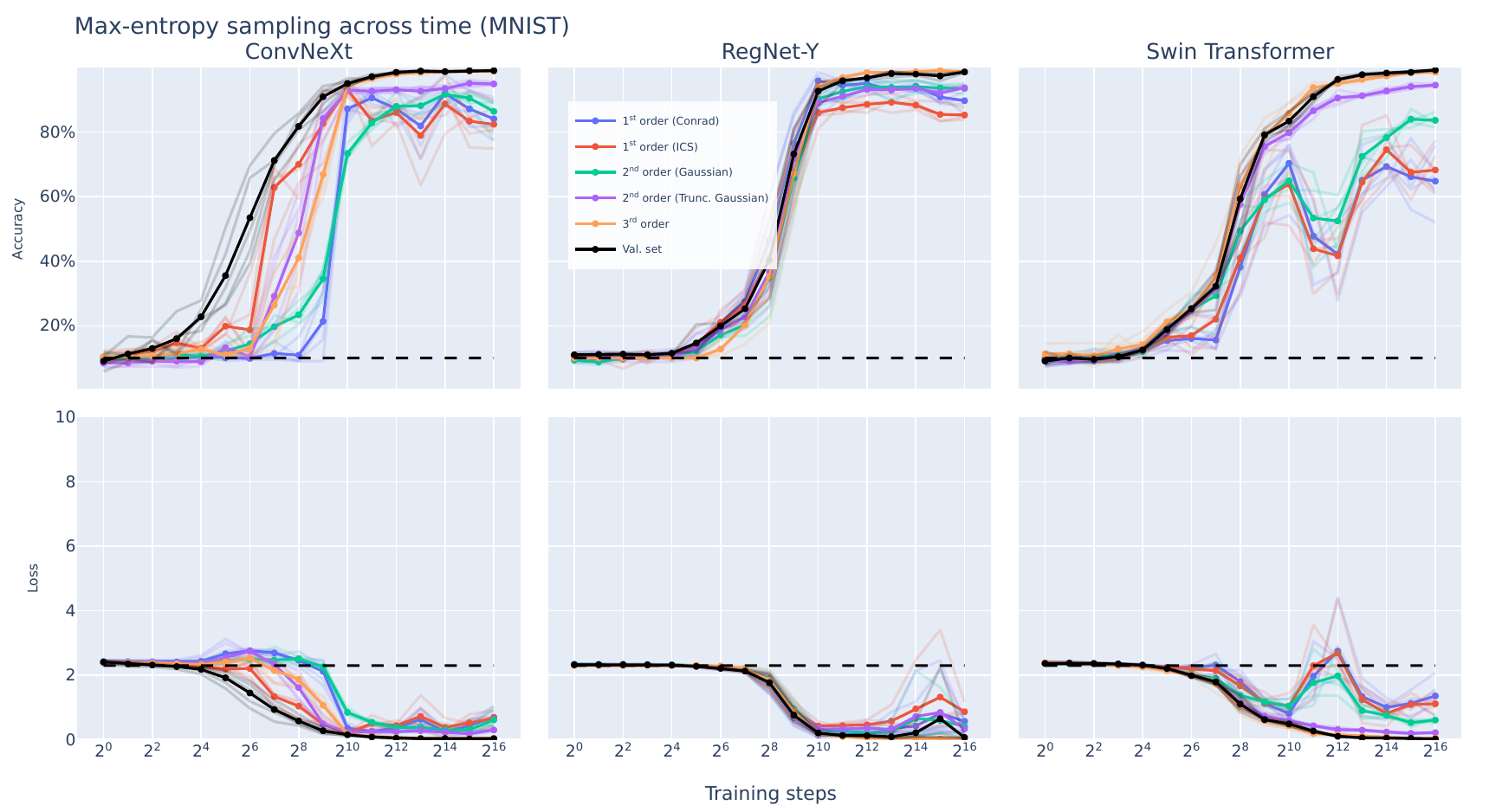}

    \caption{MNIST learning curves exhibit only a modest degree of non-monotonicity, likely because the first and second moments of the data distribution are sufficient to produce very realistic-looking samples (Fig.~\ref{fig:fake-images})}
    \label{fig:mnist}
\end{figure}

\newpage
\subsection{Street View Housing Numbers}

\begin{figure}[h!]
    \centering
    \includegraphics[width=\textwidth]{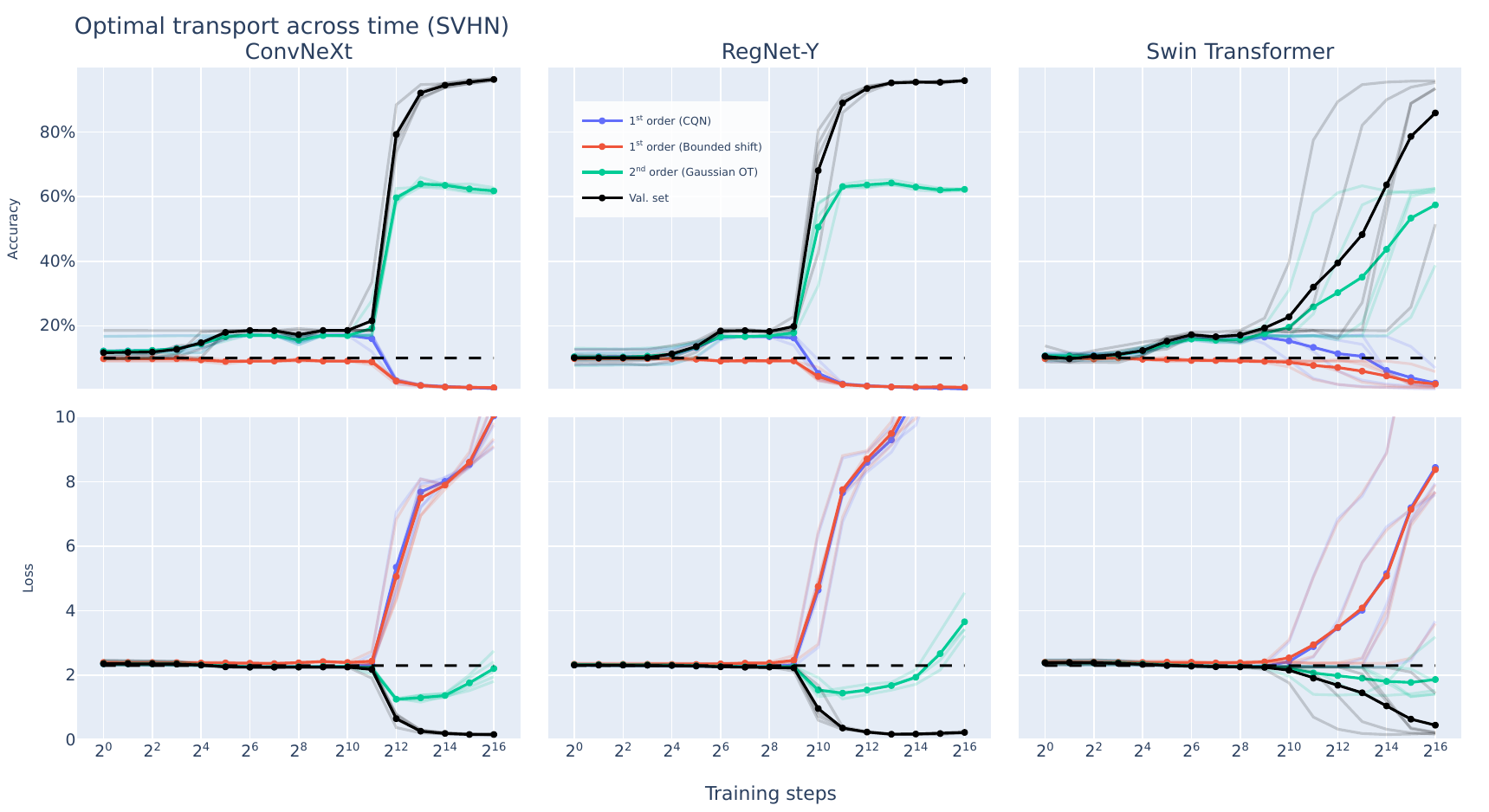}
    \includegraphics[width=\textwidth]{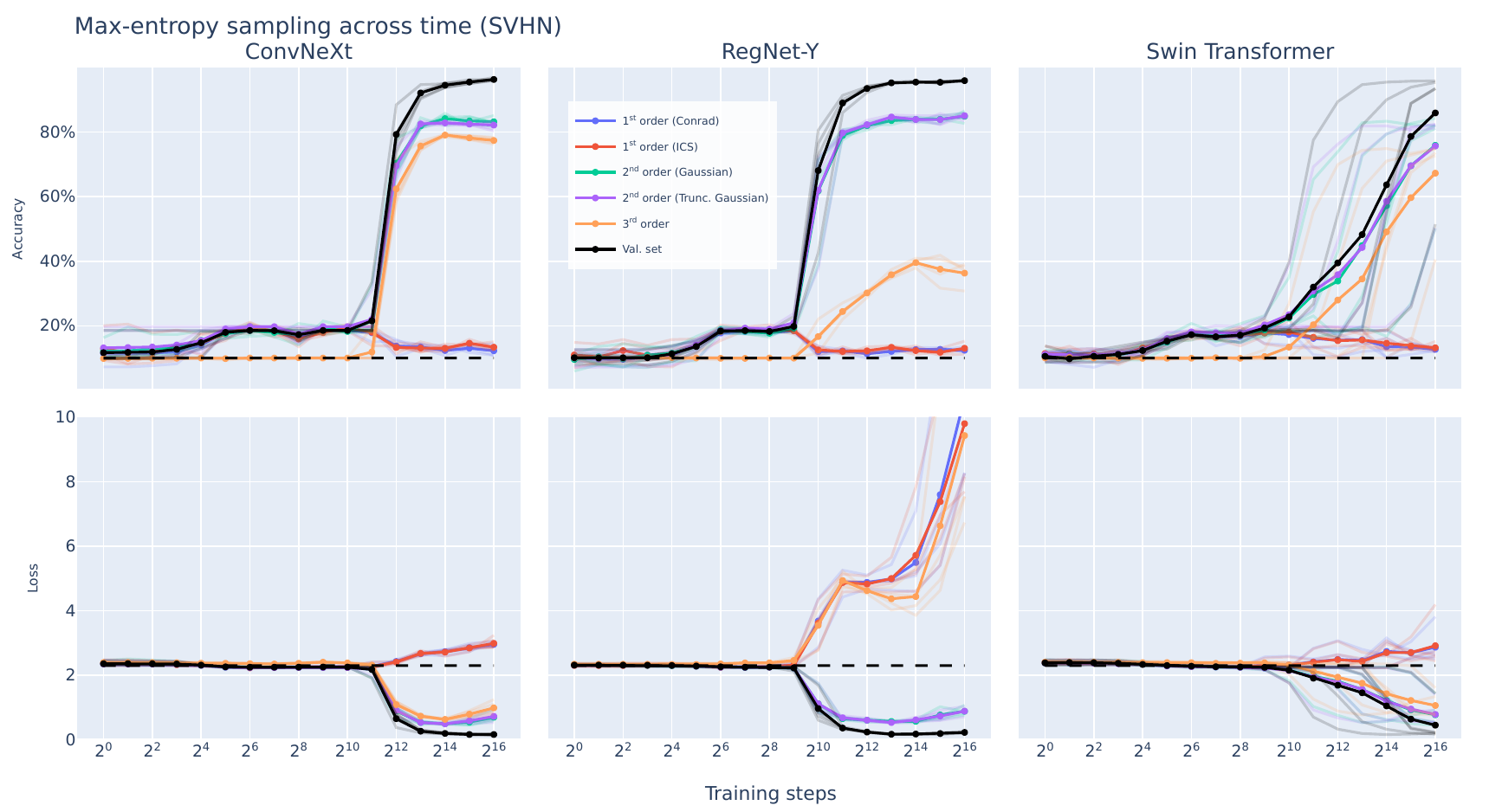}
    \caption{The Street View Housing Numbers dataset is somewhat of an outlier in that none of the models ever exceed random baseline accuracy on 1\textsuperscript{st} order synthetic images. We hypothesize this is because of the extreme diversity of colors, fonts, and background textures in SVHN, which make ``simple'' first order features less discriminative for classifying digits. We also found it necessary to use a smaller learning rate to achieve convergence on this dataset (Footnote~\ref{foot:svhn}).}
    \label{fig:svhn}
\end{figure}

\newpage
\subsection{Pythia Language Models}
\label{app:pythia}
    
\begin{figure}[h!]
    \centering
    \includegraphics[width=\textwidth]{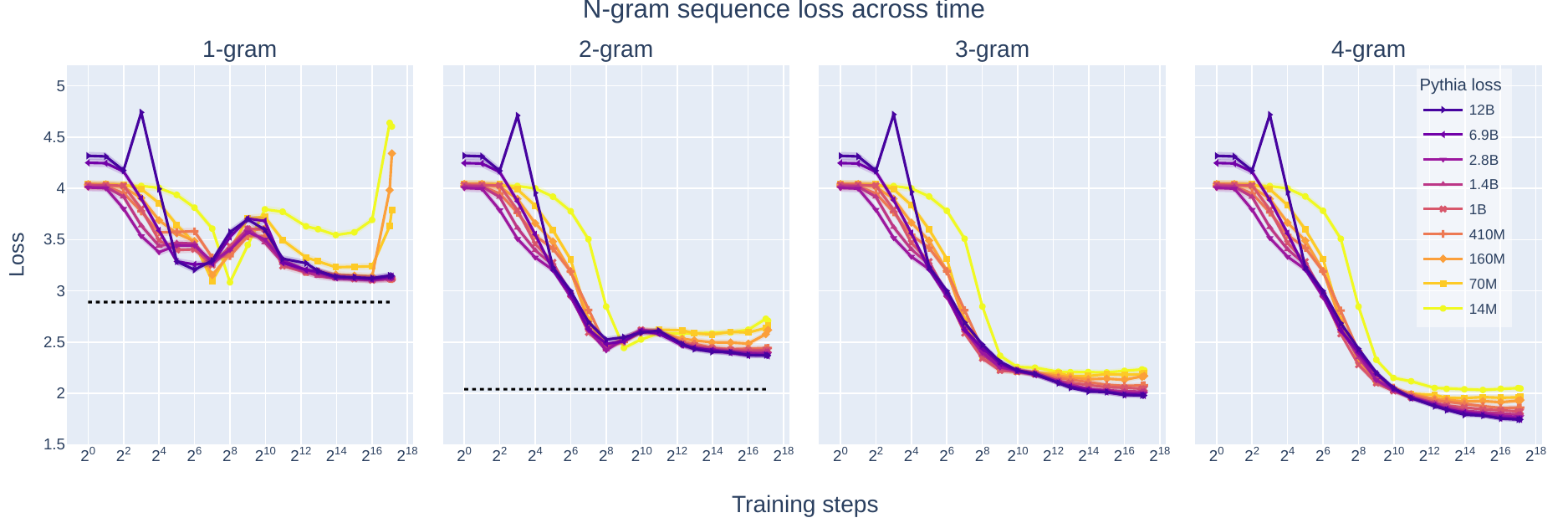}
    \includegraphics[width=\textwidth]{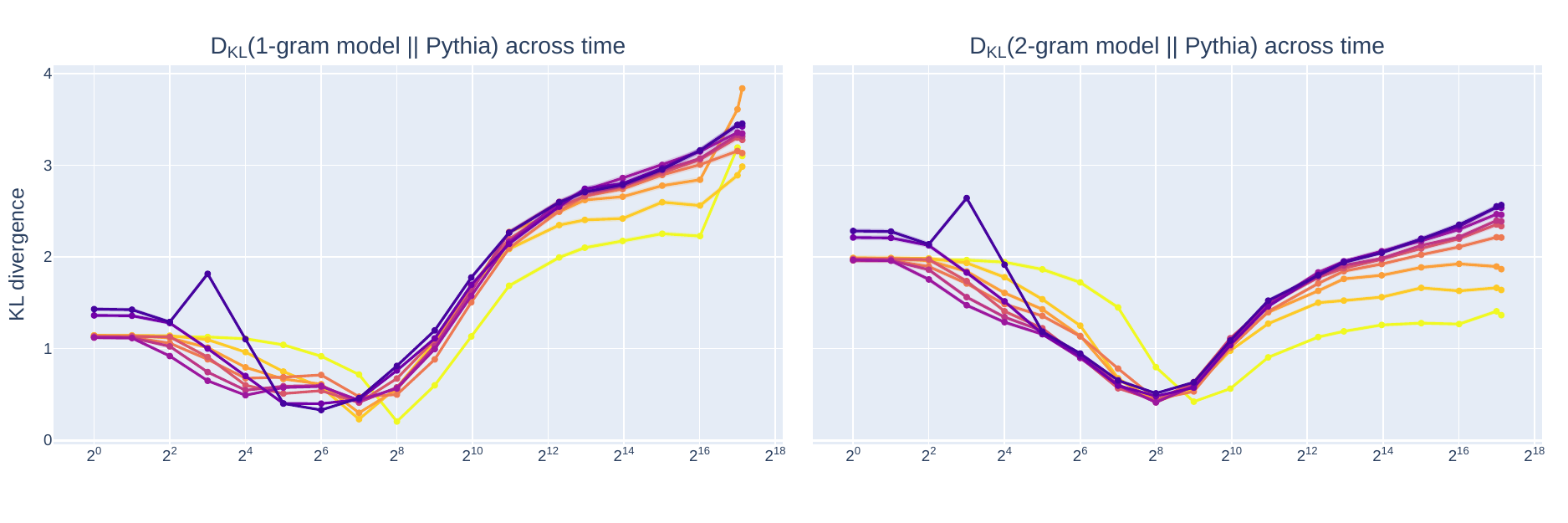}
    \caption{\textbf{(top)} Average cross-entropy loss of Pythia models evaluated on 1- through 4-gram sequences where the 3- and 4-gram models are evaluated on a subset of the Pile, \textbf{(bottom)} KL divergence between the predictions of our $n$-gram language models and the predictions of Pythia checkpoints ($N = 4096$.)}
    \label{fig:pythia-loss-across-time}
\end{figure}

\begin{figure}[h]
    \centering
    \includegraphics[width=\textwidth]{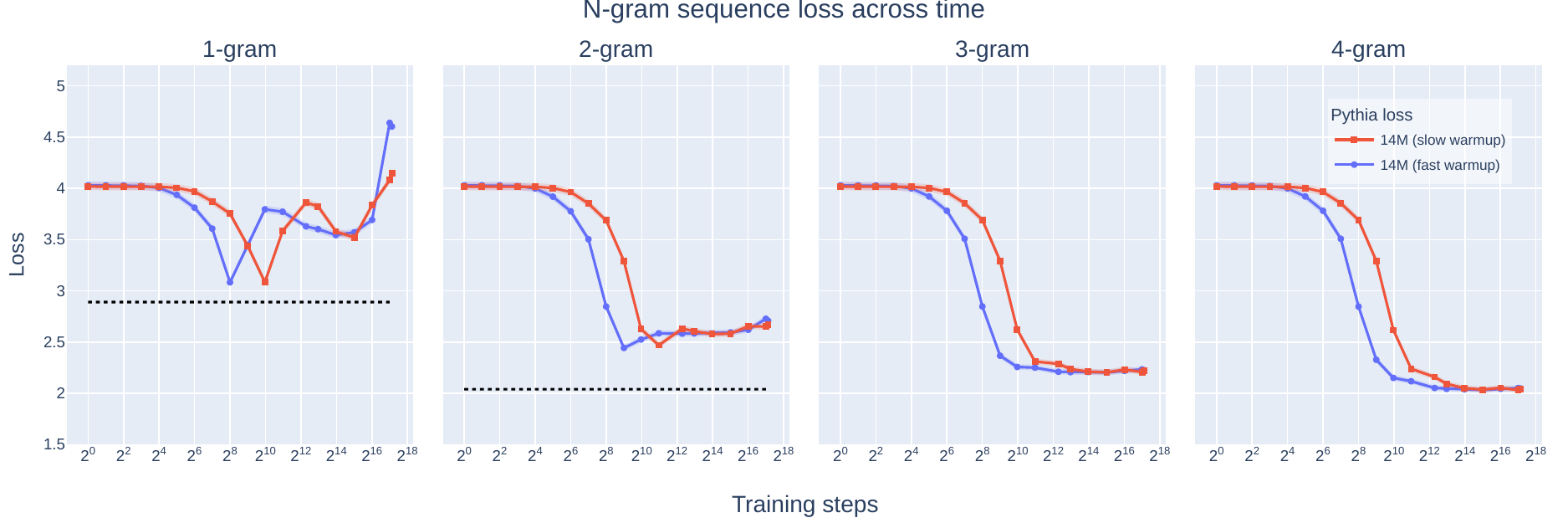}
    \vspace{0.5cm}
    \includegraphics[width=\textwidth]{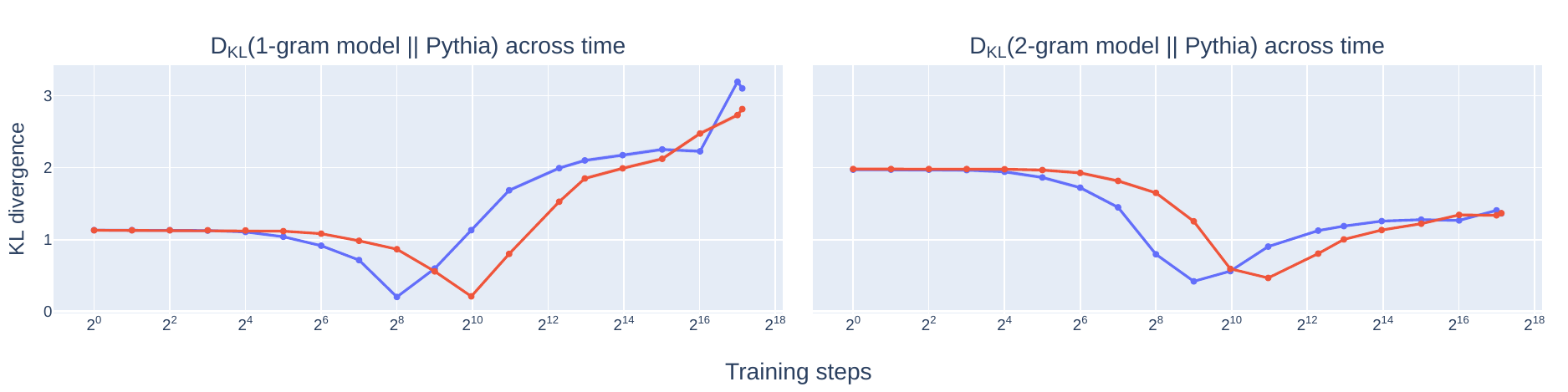}
    \includegraphics[width=\textwidth]{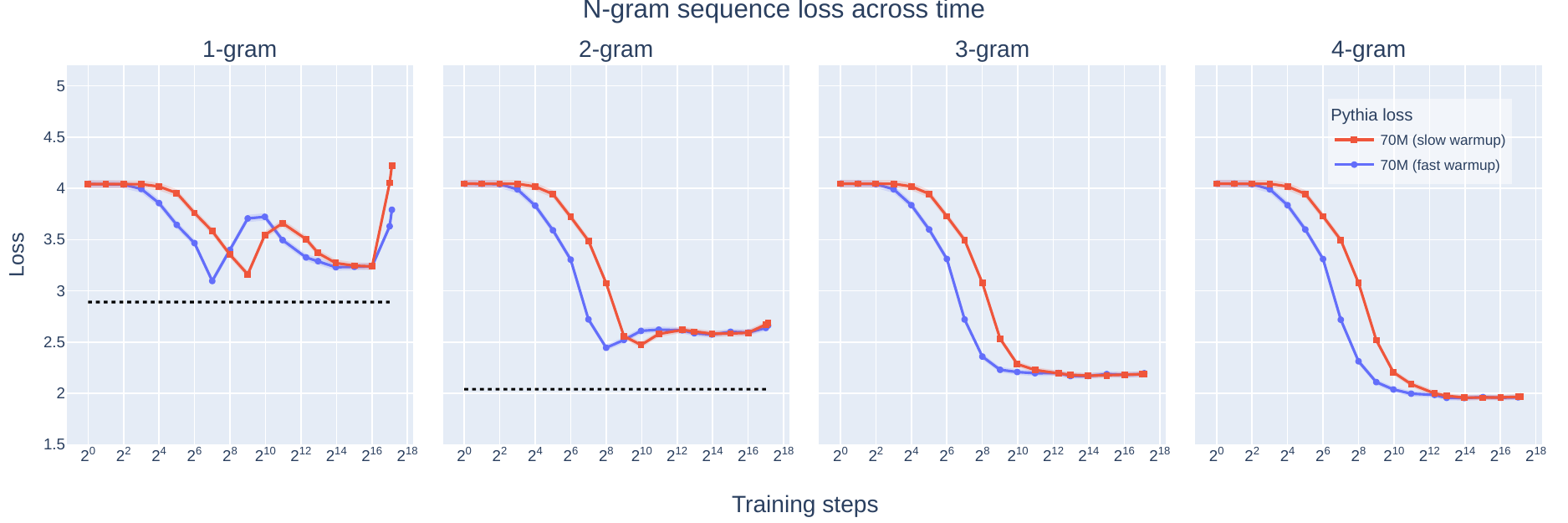}
    \includegraphics[width=\textwidth]{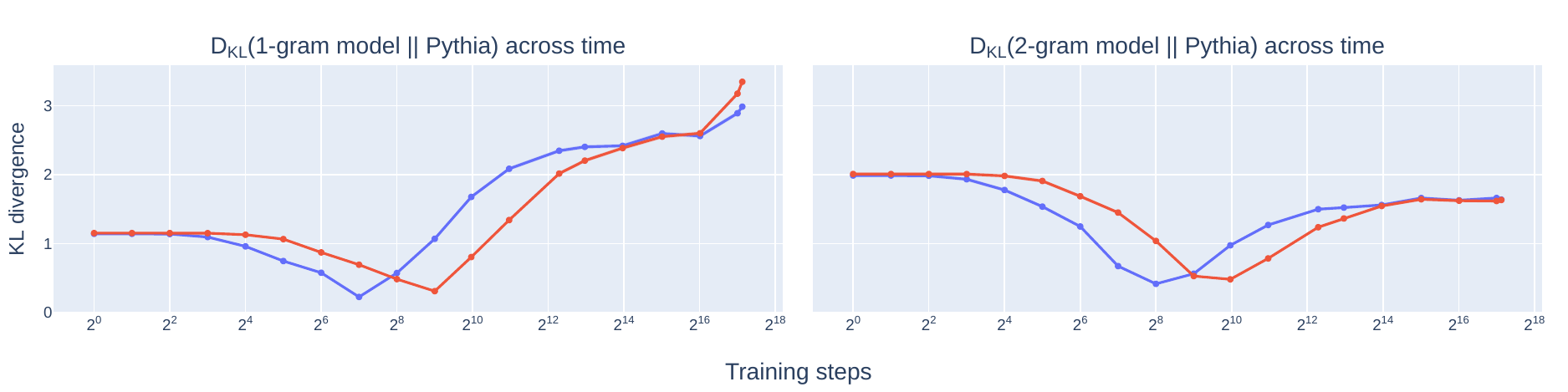}
    \caption{Effects of fast and slow learning rate warmup on n-gram sequence loss and KL divergence, Pythia 14M and 70M ($N = 4096$.)}
    \label{fig:pythia-small-warmup}
\end{figure}

\begin{figure}[h]
    \centering
    \includegraphics[width=\textwidth]{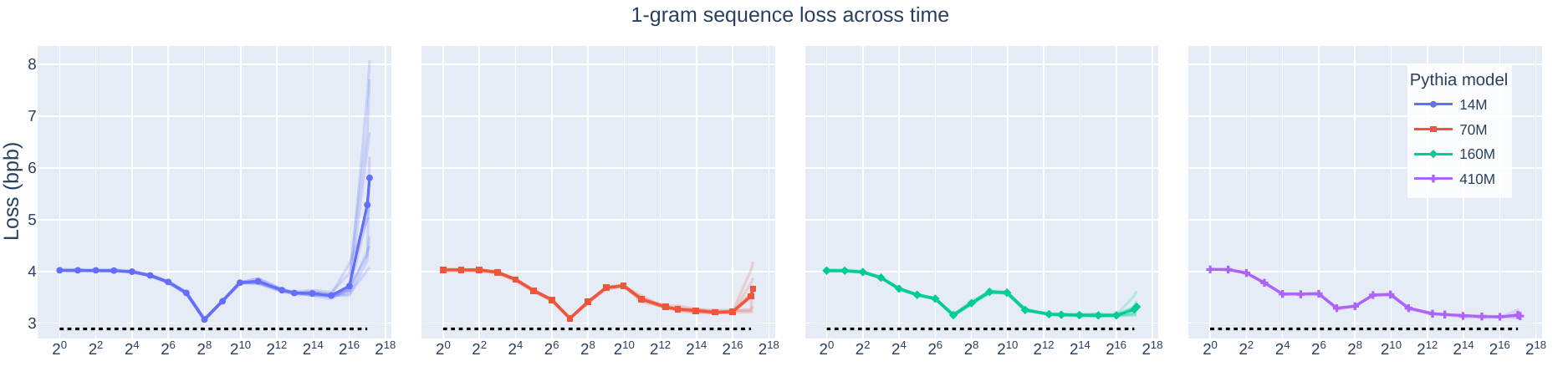}
    \includegraphics[width=\textwidth]{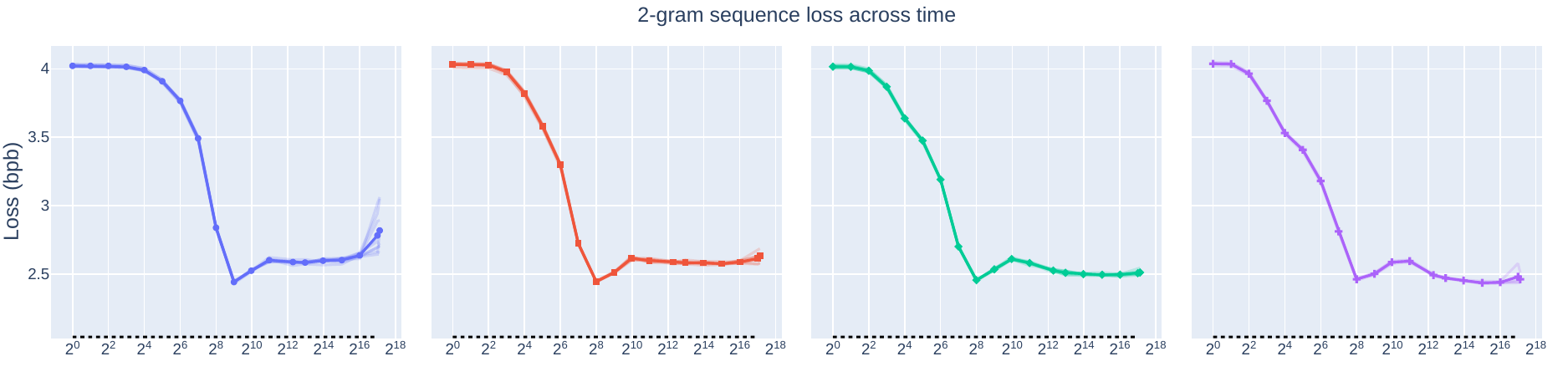}
    \includegraphics[width=\textwidth]{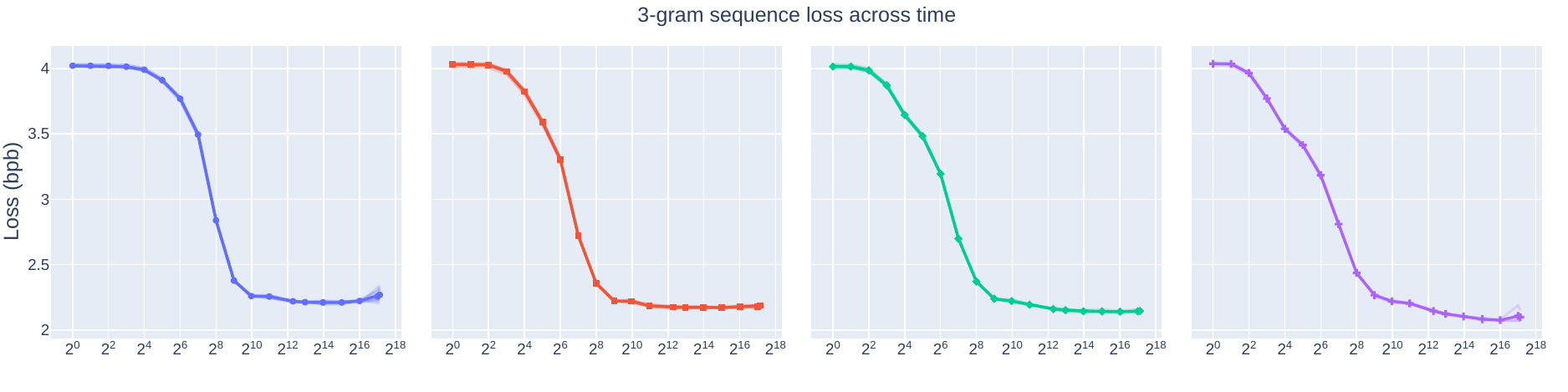}
    \includegraphics[width=\textwidth]{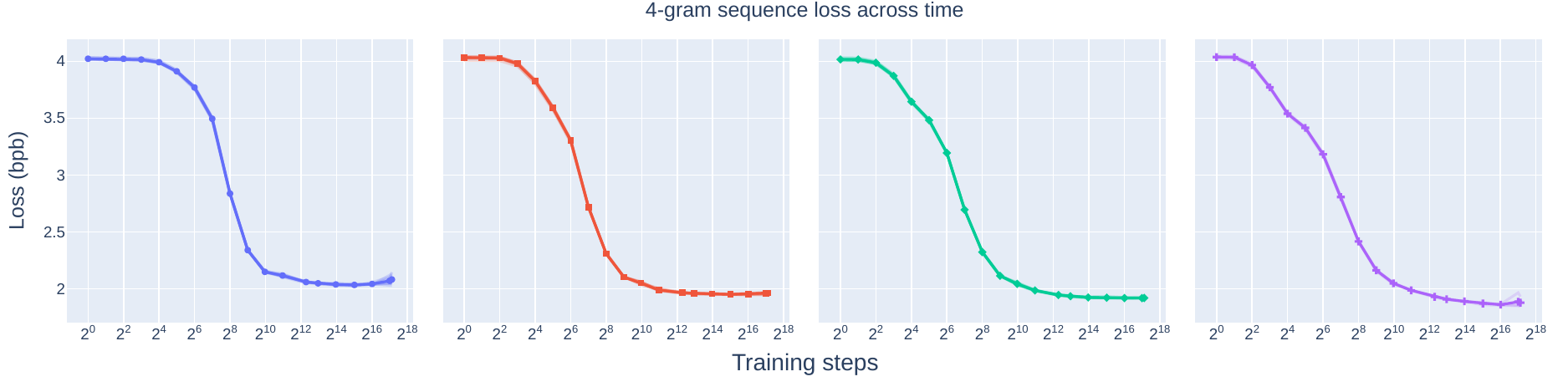}
    \caption{N-gram sequence loss over Pythia sizes and random seeds (4096 sequences sampled at each step, 9 seeds for Pythia 14M, 70M and 160M, 4 seeds for Pythia 410M.)}
    \label{fig:pythia-losses-across-time}
\end{figure}

\begin{figure}[h]
    \centering
    \includegraphics[width=\textwidth]{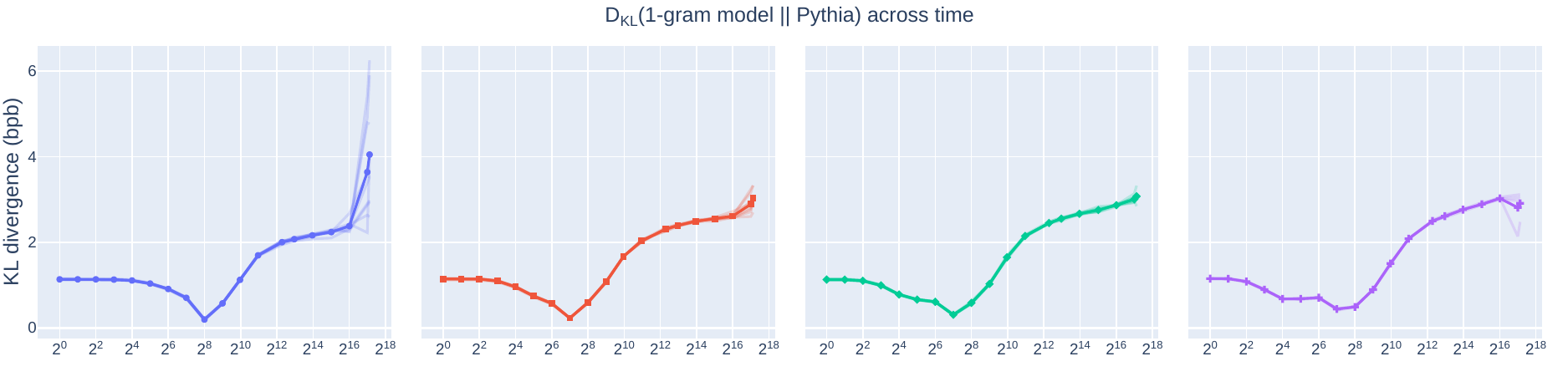}
    \includegraphics[width=\textwidth]{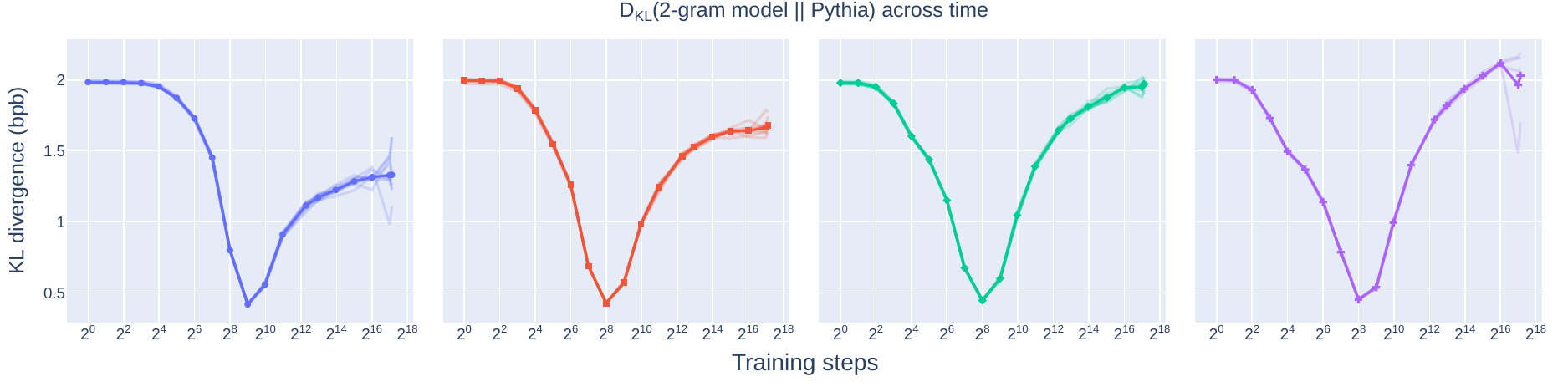}   
    \caption{N-gram KL divergence over Pythia sizes and random seeds (4096 sequences sampled at each step, 9 seeds for Pythia 14M, 70M and 160M, 4 seeds for Pythia 410M.)}
    \label{fig:pythia-kl-divs-across-time}
\end{figure}

\clearpage
\section{First order optimal transport under a boundary constraint}\label{app:mean-shift}

\begin{wrapfigure}[13]{r}{0.4\textwidth}
\vspace{-30pt}
\begin{minipage}{0.4\textwidth}
\begin{algorithm}[H]
    \caption{Optimal constrained mean shift}
    \begin{algorithmic}[1]
    \Require Input vector $\xx \in [0, 1]^n$
    \Ensure Desired mean $m \in [0, 1]$
    \State Sort the coordinates of $\xx$
    \State $\Bar{x} \gets \sum_{i = 1}^n x_i$
    \State $\yy \gets \mathbf{0}_n$
    \For{$i \in 1\ldots n$}
        \State $y_i \gets x_i + m - \Bar{x}$
        \If{$y_i > 1$}
            \State $\Bar{x} \gets \sum_{j = i}^n x_j$
            \State $y_i \gets 1$
            \State $m \gets \frac{nm - i}{n - i}$
        \EndIf
    \EndFor
    \State Put coordinates of $\yy$ in their original order\\
    \Return $\yy$
    \end{algorithmic}\label{alg:mean-shift}
\end{algorithm}
\end{minipage}
\end{wrapfigure}

We would like to surgically change the mean of a set of images while keeping their pixel intensities constrained to the range $[0, 1]$. The least-squares optimal algorithm for this task is described in Alg.~\ref{alg:mean-shift}, and we prove its correctness in the following theorem.

\begin{theorem}
    Let $\xx$ be a vector in $[0, 1]^n$ and let $m \in [0, 1]$ be a desired mean. Then the optimization problem
    \begin{align*}
        \min_{\yy \in [0, 1]^n} \|\xx - \yy\|_2^2 \quad \mathrm{s.t.}\: \frac{1}{n} \sum_{i = 1}^n y_i = m
    \end{align*}
    has a unique solution given by Algorithm 1.
\end{theorem}

\begin{proof}
    Let $\Bar{x} = \frac{1}{n} \sum_{i = 1}^n x_i$. If $\Bar{x} = m$, we immediately have the optimal solution $\yy^* = \xx$, because our constraints are already satisfied and we achieve zero loss by leaving $\xx$ unchanged.

    We can reduce the case where $m < \Bar{x}$ to the case where $\Bar{x} < m$ (or vice versa) by exploiting the reflection-symmetry of the problem. Specifically, if the solution to the analogous problem in $\xx'$ and $m'$, where $\xx' = \frac{1}{2} - \xx$ and $m' = \frac{1}{2} - m$, is $\yy'$, then the solution to the original problem is $\yy^* = \frac{1}{2} - \yy'$. This is due to the reflection-invariance of the Euclidean distance, the linearity of our mean constraint, and the fact that reflecting about $\frac{1}{2}$ cannot move points in $[0, 1]^n$ outside $[0, 1]^n$. Because of this symmetry, in what follows we will assume without loss of generality that $\Bar{x} < m$.

    Note also the optimal solution must have the property that $\forall i: x_i \leq y_i$. Assume for the sake of contradiction that $x_i > y_i$ for some $i$. Then we can improve upon this solution by increasing $y_i$ by some $\epsilon > 0$, and compensating for this by moving another entry $y_j$ for which $x_j < y_j$ closer to its original value by $\epsilon$.
    
    \paragraph{Setting up the Lagrangian.} Using the Karush-Kuhn-Tucker conditions, we encode the problem with the Lagrangian
    \begin{align}
        \mathcal{L}(\yy, \lambda, \boldsymbol{\mu}, \boldsymbol{\nu}) = \sum_{i = 1}^n (y_i - x_i)^2 + \underbrace{\lambda \Big ( \frac{1}{n} \sum_{i = 1}^n y_i - m \Big )}_{\text{mean constraint}} - \underbrace{\sum_{i = 1}^n \mu_i y_i + \sum_{i = 1}^n \nu_i (y_i - 1).}_{\text{inequality constraints}}
    \end{align}
    Differentiating $\mathcal{L}$ with respect to $y_i$ yields the stationarity condition
    \begin{align}\label{eq:first-order-opt}
        \frac{\partial \mathcal{L}}{\partial y_i} = 2(y_i - x_i) + \frac{\lambda}{n} - \mu_i + \nu_i = 0.
    \end{align}
    The KKT complementary slackness condition requires that $\mu_i y_i = 0$ and $\nu_i (y_i - 1) = 0$ for each $i$. This implies that $\mu_i$ must be zero if $y_i > 0$, and $\nu_i$ must be zero if $y_i < 1$. For each $i$ where $y_i < 1$, we can use Eq.~\ref{eq:first-order-opt} and complementary slackness to write $y_i$ as $x_i - \frac{\lambda}{2n}$.

    \paragraph{Putting it all together.}
    Assume $\xx$ and $\yy$ are written in a basis that ensures the coordinates of $\xx$ are sorted in descending order, so that $x_1 \geq x_2 \geq \ldots \geq x_n$. Our problem is invariant to permutation of indices, so this does not affect the solution.

    We can now solve for $y_1$, the final position of the largest coordinate, in the following way. Suppose that $y_1 < 1$. Then we have $\forall i : y_i < 1$, and the mean constraint can be written as $\frac{1}{n} (\sum_{i = 1}^n x_i) - \frac{\lambda}{2n} = m$. This allows us to solve for all $y_i$:
    \begin{align}\label{eq:trivial-case}
        y_i = x_i - \frac{\lambda}{2n} = x_i + m - \frac{1}{n} (\sum_{i = 1}^n x_i).
    \end{align}
    Note Eq.~\ref{eq:trivial-case} may ``overshoot'' and violate the inequality constraint $y_i \leq 1$. If it does, then we know our supposition is false and $y_1 = 1$. If it does not violate the constraint, then it must be optimal because it is also the solution to the relaxed version of this problem without the $[0, 1]$ constraint. In the latter case, we are done.

    Given that $y_1 = 1$, the subproblem of solving for $y_2, \ldots, y_n$ is a smaller instance of the original problem: the target mean for these coordinates is $m' = \frac{nm - 1}{n - 1}$. We can recursively apply this reasoning to solve for all other $y_i$. This procedure coincides with Algorithm 1.
\end{proof}

\clearpage
\section{Derivation of the Conrad distribution}

\begin{theorem}\label{thm:conrad}
    Among all distributions supported on $[0, 1]$ with desired mean $m \neq \frac{1}{2}$, the Conrad distribution with density $p(x) = \frac{b \exp(-b x + b)}{\exp(b) - 1}$ has maximum entropy, where the parameter $b \neq 0$ is chosen to satisfy the equation $m = -\frac{b - \exp(b) + 1}{b (\exp(b) - 1)}$. In the special case of $m = \frac{1}{2}$, the maximum entropy distribution is $\mathrm{Unif}(0, 1)$.
\end{theorem}

\begin{proof}
    Consider the density function $p(x) = \exp(-a - b x)$, where $a$ and $b$ are selected to satisfy normalization $\int_{[0, 1]} p(x) dx = 1$ and mean $\int_{[0, 1]} p(x) x dx = m$ constraints, and another arbitrary density $q(x)$ which satisfies the same constraints. We will show that the entropy of $q$ can be no greater than the entropy of $p$.
    \begin{align*}
        H(q) &\leq H(q, p) \tag{inequality of entropy and cross-entropy} \\
        &= \E_q [ \log p(x) ] \tag{definition of cross-entropy} \\
        &= -a - b m \tag{definition of $p(x)$ and linearity} \\
        &= H(p) \tag{QED}
    \end{align*}
    We can now analytically solve for $a$ in terms of $b$. Integrating $p(x)$ from $0$ to $1$ yields $-\frac{\exp(-a - b)}{b} + \frac{\exp(-a)}{b} = 1$. Solving for $a$ we get $a = -b + \log \big ( \frac{\exp(b) - 1}{b} \big )$, which when plugged back into the original formula gives us $p(x) = \frac{b \exp(-b x + b)}{\exp(b) - 1}$.
    
    Integration by parts yields the following formula for the mean: $\int_{[0, 1]} p(x) x dx = -\frac{b - \exp(b) + 1}{b (\exp(b) - 1)}$. We can use a root-finding algorithm such as Newton's method to solve this expression for $b$ given a desired mean $m$. Note, however, that there is a singularity in the mean formula where $b = 0$. Applying l'Hôpital's rule twice yields the limit:
    \begin{align}
        \lim_{b \rightarrow 0} -\frac{b - \exp(b) + 1}{b (\exp(b) - 1)} = \lim_{b \rightarrow 0} \frac{\exp(b) - 1}{b \exp(b) + \exp(b) - 1} = \lim_{b \rightarrow 0} \frac{\exp(b)}{b \exp(b) + 2 \exp(b)} = \frac{1}{2}.
    \end{align}
    The maximum entropy distribution supported on $[0, 1]$ with no mean constraint is known to be $\mathrm{Unif}(0, 1)$. Since it happens to have the mean $\frac{1}{2}$, we may conclude that the Conrad distribution approaches $\mathrm{Unif}(0, 1)$ as $b$ approaches $0$.
\end{proof}

\clearpage
\section{Discrete domain proofs}
\label{app:discrete-moments}

Please refer to Section~\ref{sec:discrete-domains} for context pertinent to this section.

\begin{definition}[$n$-gram statistic]\label{def:ngram-statistic}
    Let $\mathcal V^N$ be the set of token sequences of length $N$ drawn from a finite vocabulary $\mathcal V$. Given some distribution over $\mathcal V^N$, an \textbf{$\boldsymbol{n}$-gram statistic} is the probability that an $n$-tuple of tokens $(v_1, \ldots, v_n) \in \mathcal V^{n}$ will co-occur at a set of unique indices $(t_1, \ldots t_n) \in \mathbb N^{n}$.
\end{definition}

\ngrammoments*

\begin{proof}
    While it is natural to view one-hot sequences as Boolean matrices of shape $N \times |\mathcal V|$, where each row corresponds to a sequence position, we instead consider \emph{flattened} one-hot encodings in order to make use of the standard mathematical machinery for moments of random vectors.
    
    In this flattened representation, the component at index $i$ indicates whether the token at $t(i)$ is equal to the token $v(i)$, where $(t(i), v(i)) := \mathrm{divmod}(i, |\mathcal V|)$. For example, if $\mathcal V = \{ \text{``apple''}, \text{``pear''} \}$ and $N = 3$, the sequence ``apple apple pear'' will be encoded as the vector $(1, 0\:|\:1, 0\:|\:0, 1)$:
    \begin{align*}
        \underbrace{\text{apple }}_{(1, 0)} \underbrace{\text{apple }}_{(1, 0)} \underbrace{\text{pear }}_{(0, 1)} \overset{f}{\rightarrow} (1, 0, 1, 0, 0, 1).
    \end{align*}
    Now consider the moment corresponding to some arbitrary multi-index $\alpha \in \mathbb N^{N \cdot |\mathcal V|}$. For illustration, let $\alpha = (0, 2\:|\:0, 0\:|\:1, 0)$. Then the corresponding moment is
    \begin{equation}
        \E[f(x)^\alpha] =\E[\underbrace{\cancel{f(x)_1^0} f(x)_2^2}_{(0, 2} \underbrace{\cancel{f(x)_3^0} \cancel{f(x)_4^0}}_{0, 0} \underbrace{f(x)_5^1 \cancel{f(x)_6^0}}_{1, 0)}],
    \end{equation}
    where $f(x)_1^0$ denotes the first component of $f(x)$ raised to the power $0$. Since each component of $f(x)$ is a Boolean indicator for the presence or absence of a vocabulary item at a given position, we can rewrite it with \href{https://en.wikipedia.org/wiki/Iverson_bracket}{Iverson brackets}:
    \begin{align}
        &= \E \big [ [x(1) = \text{``pear''}]^2 \cdot [x(3) = \text{``apple''}] \big]\\
        &= \mathbb P[x(1) = \text{``pear''} \wedge x(3) = \text{``apple''}],
    \end{align}
    or the probability that the first and third tokens will be ``pear'' and ``apple'' respectively. Note that the exponent on the $[x(1) = \text{``pear''}]$ makes no difference here as long as it is nonzero. We can always binarize $\alpha$, replacing all nonzero values with 1, and the moment will be unchanged since any nonzero power of $\{0, 1\}$ is still $\{0, 1\}$.
    
    In general, since the coordinates are all Booleans in $\{0, 1\}$, multiplication corresponds to logical conjunction and expectation corresponds to probability:
    \begin{equation}\label{eq:conjunction}
        \E[f(x)^\alpha] = \E \Big [\prod_{i = 1}^{N} f(x)_i^{\alpha_i} \Big ] = \mathbb{P} \Big [ \bigwedge_{i \in A} x(t_i) = v_i \Big ],
    \end{equation}
    where $A$ is the set of indices in $1 \ldots |\mathcal V| \times N$ where $\alpha$ is nonzero. By Def.~\ref{def:ngram-statistic}, this probability is an $n$-gram statistic of order $k = |A|$.
    
    Conversely, we can convert any an $n$-gram statistic with tokens in $\mathcal V^{n}$ and sequence positions in $\mathbb N^n$ into a moment of $f_{\sharp}P$ by first flattening the indices, then plugging them into Eq.~\ref{eq:conjunction}. The sequence positions correspond to rows, and the tokens correspond to columns, of a one-hot matrix representation of a sequence. Here we need to multiply the row and column indices together to yield indices into the flattened vector.

    There will be infinitely many moments which correspond to any given $n$-gram, because multi-indices with components larger than one are redundant. 
\end{proof}

\embeddingmoments*

\begin{proof}
    By Thm.~\ref{thm:ngrams-are-moments}, we know that the one-hot analogues of $P$ and $Q$, i.e. $f_{\sharp}P$ and $f_{\sharp}Q$, have equal moments up to order $k$. That is, for every multi-index $\alpha \in \mathbb N^{N \cdot |\mathcal V|}$ where $|\alpha| \leq k$,
    \begin{equation}\label{eq:moment-equality}
        \E_{\mathbf{x} \sim f_{\sharp}P} [\mathbf{x}^\alpha] = \E_{\mathbf{x} \sim f_{\sharp}Q} [\mathbf{x}^\alpha].
    \end{equation}
    Now let $g : \{0, 1\}^{| \mathcal V | \times N} \rightarrow \mathbb R^{d \times N}$ be the function that multiplies each one-hot vector in a sequence by $\mathbf{E}$, returning a sequence of embedding vectors. Each side of Eq.~\ref{eq:moment-equality} is the expectation of a polynomial in the components of $\mathbf{x}$, and since $g$ is a linear map, $g(\mathbf{x})^\alpha$ is also a polynomial with the same degree.

    Now consider $(g \circ f)_{\sharp}P$ and $(g \circ f)_{\sharp}Q$, the analogues of $P$ and $Q$ in embedding space. Its moments take the form
    \begin{equation}
        \E_{\mathbf{x} \sim (g \circ f)_{\sharp}P}[\mathbf{x}^\alpha] = \E_{\mathbf{x} \sim f_{\sharp}P}[g(\mathbf{x})^\alpha].
    \end{equation}
    Because $g(\mathbf{x})^\alpha$ is a polynomial with degree $|\alpha| \leq k$, the expectation $\E_{\mathbf{x} \sim f_{\sharp}P}[g(\mathbf{x})^\alpha]$ must be a linear combination of moments of $f_{\sharp}P$ with order no greater than $k$. But by Eq.~\ref{eq:moment-equality}, all of these moments are equal between $P$ and $Q$, and hence all the moments of $(g \circ f)_{\sharp}P$ and $(g \circ f)_{\sharp}Q$ up to order $k$ must be equal.
\end{proof}

\section{Computational requirements}
\label{app:computational_requirements}

At the scales of the datasets we use in this study, both maximum entropy second order hypercube-constrained sampling and Gaussian optimal transport are extremely cheap to run. In the most expensive configuration (generating around 200K $64 \times 64$ CIFARNet images), the optimization loop takes roughly 65 seconds on a single NVIDIA L40 GPU, while requiring approximately 29 gigabytes of GPU memory. Based on hourly pricing of \$1.10 per hour from \hyperlink{https://cloud.vast.ai/?gpu_option=L40}{vast.ai}, this will cost around \$0.02 to generate a full set of synthetic CIFARNet images, with all other first and second order methods described in this paper requiring fewer computational resources than that.

However, the memory required for both hypercube-constrained sampling and Gaussian optimal transport rise with the square of the number of image features (meaning the fourth power of the image size). The compute requirements of the hypercube sampling also rise with the fourth power of the image size, while the requirements for Gaussian optimal transport rise with the sixth power. This means our methods can quickly become computationally infeasible with larger image sizes, which is why we limit ourselves to at most $64 \times 64$ images in this study.

Additionally, the maximum entropy \textit{third order} hypercube-constrained sampling is much more expensive than the second order methods, since the size of the statistic tensor grows as $O(d^{order})$. This means the coskewness tensor for CIFARNet images has dimensions $12288 \times 12288 \times 12288$. This would require nearly eight terabytes to store in full precision, which exceeds the memory capacity of our computing hardware by a significant degree. 

We therefore want to generate fake data that matches CIFARNet's coskewness statistics, without ever computing those statistics in full. Each step of our optimization process for generating third order fake thus only computes the coskewness statistics along matching length $l$ slices of coskewness tensors of the fake and real data, meaning we only need to store two $12288 \times 12288 \times l$ tensors at each step of optimization.

CIFARNet is the most expensive dataset to imitate with maximum entropy third order hypercube-constrained sampling, as matching its first, second, and third order statistics at the same time takes significant optimization effort. We currently use 10,000 optimization steps per class, taking a total of 36 hours on a single NVIDIA A40 GPU. Using an hourly price of \$0.403 from \hyperlink{https://cloud.vast.ai/?gpu_option=A40}{vast.ai}, this would cost roughly \$14.5.



\end{document}